%% file: acl_latex.tex
\pdfoutput=1

\documentclass[11pt]{article}

\usepackage[preprint]{acl}

\usepackage{times}
\usepackage{latexsym}

\usepackage[T1]{fontenc}

\usepackage[utf8]{inputenc}

\usepackage{microtype}

\usepackage{inconsolata}

\usepackage{graphicx}
\usepackage[utf8]{inputenc} %
\usepackage[T1]{fontenc}    %
\usepackage{booktabs}       %
\usepackage{amsfonts}       %
\usepackage{nicefrac}       %
\usepackage{microtype}      %

\usepackage{amsmath}
\usepackage{amssymb}
\usepackage{mathtools}
\usepackage{amsthm}

\usepackage{xfrac}

\usepackage{subcaption}

\usepackage{algorithm}
\usepackage[noend]{algpseudocode}

\algrenewcommand\algorithmicrequire{\textbf{Input}}
\algrenewcommand\algorithmicensure{\textbf{Output}}
\algnewcommand{\IIf}[2]{\State\algorithmicif\ #1\ \algorithmicthen\ #2}
\algnewcommand{\IElse}[1]{\State\algorithmicelse\ #1}
\algnewcommand{\IElseIf}[2]{\State\algorithmicelse\ \algorithmicif\ #1\ \algorithmicthen\ #2}
\makeatletter
\patchcmd{\ALG@doentity}{\item[]\nointerlineskip}{}{}{}
\ifthenelse{\equal{\ALG@noend}{t}}%
  {\algtext*{EndNIElse}}
  {}%
\ifthenelse{\equal{\ALG@noend}{t}}%
  {\algtext*{EndNIElseIf}}
  {}%
\makeatother

\DeclareFontFamily{U}{matha}{\hyphenchar\font45}
\DeclareFontShape{U}{matha}{m}{n}{ <-6> matha5 <6-7> matha6 <7-8>
matha7 <8-9> matha8 <9-10> matha9 <10-12> matha10 <12-> matha12 }{}
\DeclareSymbolFont{matha}{U}{matha}{m}{n}
\DeclareFontFamily{U}{mathx}{\hyphenchar\font45}
\DeclareFontShape{U}{mathx}{m}{n}{ <-6> mathx5 <6-7> mathx6 <7-8>
mathx7 <8-9> mathx8 <9-10> mathx9 <10-12> mathx10 <12-> mathx12 }{}
\DeclareSymbolFont{mathx}{U}{mathx}{m}{n}
\DeclareMathDelimiter{\liv} {4}{matha}{"76}{mathx}{"30}
\DeclareMathDelimiter{\riv} {5}{matha}{"77}{mathx}{"38}

\usepackage{wrapfig}
\usepackage[capitalise]{cleveref}

\usepackage{bm}
\usepackage{dsfont}

\input{tikz_config}

\usepackage{thm-restate}

\newtheorem{theorem}{Theorem}[section]

\theoremstyle{definition}
\newtheorem{definition}[theorem]{Definition}
\newtheorem{assumption}[theorem]{Assumption}
\newtheorem{problem}[theorem]{Problem}
\theoremstyle{remark}

\newif\ifcomments
 \commentstrue
\ifcomments
    \newcommand{\kareem}[1]{\textbf{\textcolor{magenta}{{Kareem: #1}}}}
    \newcommand{\guy}[1]{{\bfseries\color{red}Guy: #1}}
    \newcommand{\renato}[1]{{\bfseries\color{cyan}Renato: #1}}
    \newcommand{\benjie}[1]{{\bfseries\color{green}Benjie: #1}}
    \newcommand{\honghua}[1]{{\bfseries\color{orange}Honghua: #1}}
\else
    \newcommand{\kareem}[1]{}
    \newcommand{\guy}[1]{}
    \newcommand{\renato}[1]{}
    \newcommand{\benjie}[1]{}
    \newcommand{\honghua}[1]{}
\fi

\input{macros}

\title{Where is the signal in tokenization space?}

\author{Renato Lui Geh, \quad Honghua Zhang, \\ \textbf{Kareem Ahmed,} \quad  \textbf{Benjie Wang,} \quad \textbf{Guy Van den Broeck}\\
        University of California, Los Angeles\\\texttt{\{renatolg,hzhang19,ahmedk,benjiewang,guyvdb\}@cs.ucla.edu}}

\begin{document}
\maketitle
\begin{abstract}
Large Language Models~(LLMs) are typically shipped with tokenizers that \emph{deterministically} encode text into so-called \emph{canonical} token sequences, to which the LLMs assign probability values.
One common assumption is that the probability of a piece of text is the probability of its canonical token sequence.
However, the tokenization of a string is not unique: e.g., the \llama{} tokenizer encodes \texttt{Tokens} as \texttt{[Tok,ens]}, but \texttt{[Tok,en,s]} also represents the same text.
In this paper, we study non-canonical tokenizations. %
We prove that, given a string, it is computationally hard to find the most likely tokenization for an autoregressive LLM, as well as to compute the marginal probability over all possible tokenizations.
We then show how the marginal is, in most cases, indistinguishable from the canonical probability.
Surprisingly, we then empirically demonstrate the existence of a significant amount of signal hidden within tokenization space.
Notably, by simply aggregating the probabilities of non-canonical tokenizations, we achieve improvements across a range of LLM evaluation benchmarks for a variety of architectures, including transformers and state space models.

\end{abstract}

\section{Introduction}

Autoregressive large language models (LLMs) generate text by predicting the next word sequentially.
A crucial yet often overlooked step in this process is \emph{tokenization}, whereby each
word is broken down into subwords.
It allows the model to generate text beyond what it was trained on, enabling open-vocabulary generation.
However, this approach also introduces a significant challenge: a given string can be tokenized in exponentially many ways~(\Cref{fig:num_tokenizations}). For example, this paper's abstract can be tokenized in more than $10^{267}$ ways under the \llama{} vocabulary \cite{llama2}.

\begin{figure}[t]
    \input{tikz/exp_growth}
    \caption{\textbf{Exponential growth of the number of tokenizations.} The (log-scale) $y$-axis shows the number of tokenizations as a function of the sentence length.}
    \label{fig:num_tokenizations}
\end{figure}
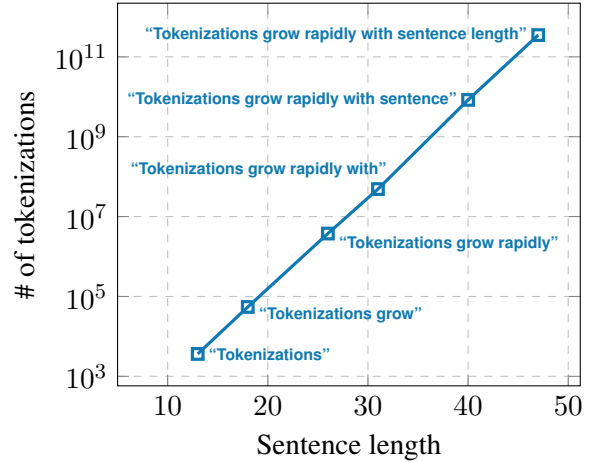

While a given string can be tokenized in multiple ways, at inference time, almost all 
successful modern LLMs utilize a fixed, or \emph{canonical}, tokenization: a deterministic, rule-based mapping from text to token sequences \citep{gage1994BPE}.
Consequently, it has become commonplace to use this token sequence as a proxy for the underlying text. In particular, the probability of the token sequence is often used in place of the probability of the text (e.g. for evaluation metrics like perplexity),
even though these quantities are not necessarily equal.
To complicate matters, some language models are pretrained with 
\emph{stochastic tokenizations}~\citep{kudo2018subword, provilkov2020bpe}, exposing them to multiple ways of tokenizing the same string, with the hope of obtaining models with a more developed
understanding of the compositionality of words.

In this paper, in the context of modern LLMs, we ask whether %
non-canonical tokenizations of a string
can provide additional signal \emph{at inference time}, which would be lost by considering just the canonical tokenization. To this end, we investigate two natural alternatives: finding the \emph{most likely tokenization} and \emph{marginalizing over tokenizations}.
For example, one natural way to answer a multiple-choice question is to choose the answer with the highest probability conditioned on the question~\cite{hellaswag}; instead of always assuming the canonical tokenizations of the answers, one could compare answers based on the probability of their most likely tokenizations, or alternatively the marginal probability of all possible tokenizations. 

We first study the problem of finding the most likely tokenization and show that it is NP-hard under some mild assumptions. As such, we propose an anytime branch-and-bound algorithm to approximate the most likely tokenization. We find that, for text lengths where the branch-and-bound strategy is practical, the canonical tokenization is usually the most likely one.

Then we ask the question of whether there is a significant amount of probability mass concentrated on tokenizations \emph{other} than the canonical. We first observe that as we sample token sequences of varying length from the LLM distribution \emph{unconditionally}, the proportion of canonical tokenizations decreases significantly as the sequence length increases.
To further investigate this phenomenon, ideally one would need to compute the marginal probability of all tokenizations for a given string, which we show to be \#P-hard. Hence, we implement an importance sampling estimator for the marginal probability. Surprisingly, despite the extremely large number of non-canonical tokenizations, we empirically find that the estimated marginal probability is usually very close to the canonical tokenization's probability.

This raises our last question: does the complete tokenization space add any meaningful signal at all, in addition to the canonical tokenization alone? Remarkably, we show that, even for the cases where there is little probability mass on non-canonical tokenizations, they seem to carry \emph{some} meaningful signal.
Specifically, we show that for \gemma{}-2B \cite{gemma}, \llama{}-7B \cite{llama2} and \mamba{}-130M \cite{mamba}, by employing ensemble strategies for weighting different tokenizations at inference time, we achieve significant performance improvements on challenging LLM evaluation benchmarks.%

\paragraph{Contributions.} In summary, we show that: 
(i)~while it is tempting to consider computing the marginal probability of a string, this quantity is \#P-hard to compute;
(ii)~in fact, even computing the probability of the most likely tokenization is NP-hard; and
(iii)~while in most cases the marginal probability of a string is practically the same as the canonical probability, non-canonical tokenizations seem to provide some signal to downstream tasks, to the point that we achieve consistent improvement across a range of open source models on Q\&A datasets.

\section{Related Work}

Many previous works have explored tokenization strategies within the LLM pipeline, and the (often undesirable) inductive biases they may introduce: for example, in introducing unfairness between languages~\cite{petrov2023language}, gender bias~\cite{ovalle24}, and in performing arithmetic~\cite{singh2024tokenization}.
Some recent works have avoided the many downsides of tokenization by employing byte-level models, but either suffer from slow decoding due to longer sequences \cite{megabyte}, or rely on token-level models for more efficient generation \cite{mambabyte}.
To overcome the limitations of tokenization, prior works have examined (approximately) marginalizing over the distribution of possible token sequences~\cite{buckman18, cao21,chirkova23}. In this work, we analyze modern LLMs and consider multiple strategies for extracting information from tokenization space; finding that, contrary to prior belief, the signal is present not in the most-likely tokenization or (approximated) marginals, but rather in a mixture of canonical and non-canonical tokenizations.

\section{An LLM Induces a Distribution over Tokenizations}

Let $\sentence = (\character_1, \character_2, \ldots)$ denote a string (a sequence of characters).
A vocabulary $\vocab$ is a set of strings that represent subwords, or \emph{tokens}.
A \emph{token sequence} \wrt a vocabulary $\vocab$ is a sequence $\tokens = (\token_1, \token_2, \dots)$ where each $\token_i \in \vocab$.
A \emph{tokenization} of string $\sentence$  \wrt a vocabulary $\vocab$  is a token sequence \wrt $\vocab$ such that the concatenation of all tokens is equal to $\sentence$.
Simply put, a tokenization breaks down a string into substrings, each
recognized by the vocabulary.
The substrings are ordered by their position in the original string.
We write $\tokens \models_\vocab \sentence$ to denote that token sequence $\tokens$ is a tokenization of string $\sentence$ \wrt the vocabulary $\vocab$, sometimes omitting $\vocab$ when meaning is clear.

An autoregressive LLM $\llmdist$ defines a conditional probability distribution $\llmdist(\token_i | \token_1, \dots\token_{i-1})$ over tokens from its vocabulary $\vocab$.
Thus, an LLM \emph{induces a distribution over tokenizations of a given string}.

\begin{definition}[Induced Tokenization Distribution] \label{def:induced}
    Let $\sentence$ be a string, $\tokens$ a token sequence, and $\llmdist$ an LLM over vocabulary $\vocab$. Then, the tokenization distribution induced by~$\llmdist$~is
    \begin{equation*}
        \llmdist(\tokens,\sentence) = 
        \begin{cases}
             \prod_{i=1}^{|\tokens|} \llmdist\left(\token_i | \token_1,\dots, \token_{i-1}\right) \text{~if~} \tokens \models_\vocab \sentence,\\
             0 \quad \text{otherwise.}
        \end{cases}
    \end{equation*}
\end{definition}

\begin{figure}[t]
    \begingroup
    \input{tikz/mdd}
    \endgroup
    \centering
    \caption{\textbf{Multi-valued decision diagram for the tokenization of \texttt{Bird}}. The square is a terminal node.}
    \label{fig:Bird}
\end{figure}
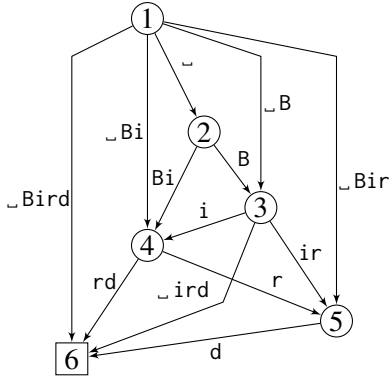

Most modern LLMs make use of tokenizers based on Byte-Pair Encoding (BPE)~\citep{gage1994BPE},
whereby the token vocabulary is initialized with the character vocabulary, and a \emph{merge table} is
initialized to be empty.
The method then iteratively counts all pairs of tokens and merges the most frequent pair into a new token.
    This new token is added to the vocabulary and the merge rule is added to the merge table.
This process is repeated until the desired vocabulary size is reached.
The resulting merge table specifies which tokens are to be merged into larger tokens, as well as the
priority of these merges.
In this way, it defines a \emph{canonical tokenization} procedure as follows:
first, a string is split into its constituent characters,
then, the pair of adjacent tokens with the highest priority merge rule is combined.
This is repeated until no further merge rules from the table are applicable.

\begin{algorithm}[t]
    \caption{\textsc{Compile}}\label{alg:compilation-mdd}
    \begin{algorithmic}[1]
        \Require String $\sentence$, vocabulary $\mathcal{V}$, last token $v$
        \Ensure MDD of all tokenizations of $\sentence$
        \State Initialize memoization $\mathcal{M}:\mathbb{N}\times\mathcal{V}\to\text{MDD}$
        \IIf{$\sentence$ is empty}{\textbf{return} nothing}
        \State Initialize node $N$
        \For{each token $t\in\mathcal{V}$}
            \If{$\sentence$ starts with $t$}
                \IIf{$(|\sentence|,t)\in\mathcal{M}$}{$C\gets\mathcal{M}(|\sentence|,t)$}
                \State \textbf{else} $C\gets\textsc{Compile}(\sentence_{|t|+1:|\sentence|},\mathcal{V},t)$
                \State Add edge from $N$ to $C$ labelled with $t$
                \State Memoize $\mathcal{M}(|\sentence|,t)\gets C$
            \EndIf%
        \EndFor%
        \State \textbf{return} $\mathrm{N}$
    \end{algorithmic}
\end{algorithm}

BPE dropout \cite{provilkov2020bpe} introduces an additional step during training: when tokenizing a word, merge rules are dropped with some probability. After this dropout phase, merges proceed the same way as BPE.
This dropout phase acts as a regularization method that provides robustness to input noise.
It also means that language models trained with BPE dropout should assign more mass to non-canonical tokenizations.

At inference time, for a string $\sentence$, the tokenizer outputs the canonical tokenization $\tokens^\ast$ (without any dropout),
which is then evaluated by the LLM to get the \emph{canonical probability}
$\llmdist(\tokens^\ast,\sentence)$.
Note that this probability is one of an exponential number of tokenization probabilities for a particular string.
In fact, one can compile a Multi-valued Decision Diagram (MDD) \cite{lee59} that represents this combinatorial space for a given string tractably by decomposing and reusing subsequences.
This data structure allows one to compute the total number of tokenizations in linear time in the number of edges of the diagram. \Cref{alg:compilation-mdd} shows how to compile an MDD from a string and \Cref{fig:Bird} shows an example of an MDD compiled from the string \texttt{Bird}.
Each node in the diagram corresponds to a position in the string, and edges from node $i$ to node $j$ are labelled with the corresponding token $\sentence_{i:j} = (x_i,x_{i+1},\dots,x_j)$. Every path going from the root to a terminal node is a tokenization of $\sentence$.

Given what we know so far, a question naturally arises: since we can tractably represent all tokenizations as an MDD, and given that the number of possibilities is exponential, can we efficiently compute the most likely tokenization of a string?
And perhaps more interestingly, is it the canonical one, as often is assumed in practice?

\section{Computing the Most Likely Tokenization is Hard ...}

We begin by noting that there exist simple distributions where finding the most likely tokenization can be done efficiently. For example, if we annotate the edges in an MDD with probabilities, that gives us a tokenization distribution where the most likely tokenization is simply the MDD path with the highest probability. By carefully modifying the MDD, it even becomes possible to efficiently compute the most likely tokenization induced by a bi-gram distribution, where each token depends only on the previous one \cite{lattice,pcsurvey}.

For more complex autoregressive language models, however, we unfortunately show that computing the most likely tokenization is computationally hard.
We formalize this as follows.

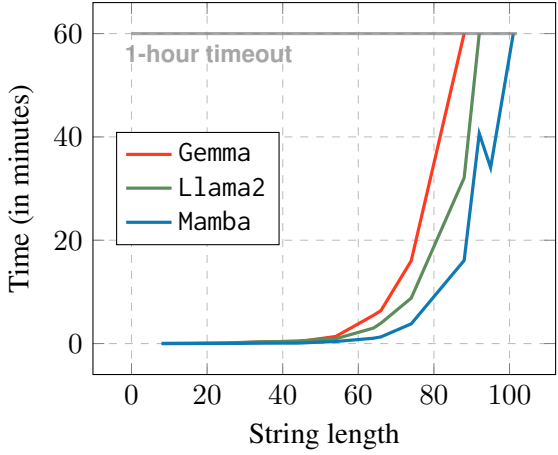
\begin{figure}[t]
    \input{tikz/branch_bound}
    \caption{\textbf{Run-time for branch-and-bound over tokenizations.} The search grows exponentially with the number of characters in the string.
    }\label{fig:branch-and-bound}
\end{figure}

\begin{restatable}[Most-Likely Tokenization]{problem}{prbMPE}\label{prob: mpe}
    Let $\tokens$ denote a token sequence. Given a string $\sentence$, an autoregressive LLM $\llmdist$, and a threshold $\epsilon > 0$, the most likely tokenization problem is deciding whether
    \begin{equation*}
        \max_{\tokens} \llmdist(\tokens,\sentence) > \epsilon.
    \end{equation*}
\end{restatable}

\begin{restatable}[]{theorem}{thmMPE} \label{thm:hardmpe}
    The most-likely tokenization problem is NP-complete.
\end{restatable}
\begin{proof}(Sketch)
    The proof is by reduction from the 3-SAT Boolean satisfiability problem~\citep{karp2010reducibility}. We encode the Boolean variables as possible tokenizations of substrings such that there is a correspondence between the probability of the most likely tokenization and the existence of a satisfiying assignment. The full proof is in Appendix~\ref{apxsec:hardness}.
\end{proof}

\begin{figure}[t]
    \input{tikz/hist_tokens}
    \vspace{-0.14cm} %
    \caption{\textbf{Distribution of tokenizations for the word \texttt{Tokens}.} An overwhelming probability mass is on the canonical tokenization, with (an exponential number of) others sharing a miniscule portion of probability.}%
    \label{fig:mar-vs-can}
\end{figure}
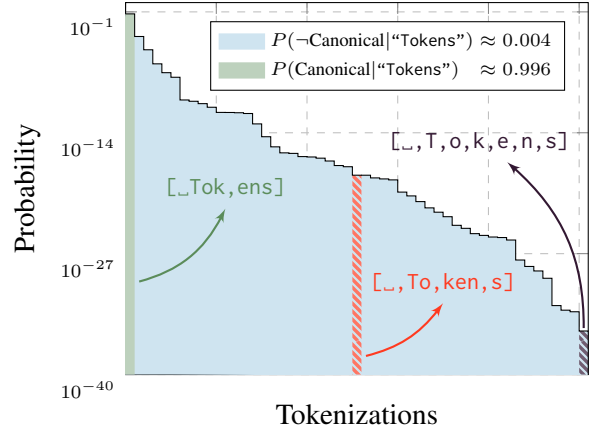

Given the LLM training regime, a reasonable assumption in practice is that the canonical tokenization is (close to) the most likely tokenization.
To empirically verify this claim, we devise a branch-and-bound algorithm to search through the MDD in order to find some tokenization whose probability is higher than the canonical.
We do so by setting the initial lower bound as the canonical probability, and then pruning paths whose partial probability is below this bound.
We set a time budget of 1-hour, after which the search returns the best tokenization at that point.
As expected, we find that branch-and-bound is quickly overwhelmed by the number of tokenizations as the string length grows. Finding the most likely tokenization this way rapidly becomes intractable for longer strings.

\Cref{fig:branch-and-bound} shows the branch-and-bound search time across three LLM architectures for the string ``Language models typically tokenize text into subwords, utilizing a learned deterministic set of merge rules to aggregate tokens.'' 
We gradually insert new words and re-run search to visualize its scalability.
Branch-and-bound always returns the canonical tokenization as the best candidate, despite the exponential number of possible candidates.

Not only does the canonical tokenization seem to be the most likely one for shorter text, but it also often is overwhelmingly so.
\Cref{fig:mar-vs-can} shows the tokenization  distribution for the word \texttt{Tokens} under the \llama{} model; there are 52 tokenizations, with the canonical taking most of the mass.

Even though canonical seems to take up a majority of the probability mass in these cases, if we look at generated text from these LLMs, a sizable percentage of the (unconditionally) generated tokenizations are non-canonical.
\Cref{fig:canon-perc} shows canonicity as a function of the number of tokens generated by the language model.
As generated text grows larger, the probability of generating non-canonical tokenizations also grows.
This is surprising, as it is seemingly in contradiction to earlier evidence.
It turns out that, if we investigate these non-canonical generated sequences in more depth, we find that a large majority of such cases are non-English, with a large portion consisting of code and languages that utilize unicode characters.
It is nevertheless interesting that some non-canonical tokenizations are indeed more likely than their canonical counterparts, with some even in grammatically correct English.
For example, the string $\sentence=\texttt{\implws tongueless}$ (with \texttt{\implws} denoting a whitespace) is canonically tokenized as $\tokens^\ast=\texttt{[\implws tong,uel,ess]}$ by $\gemma$, with $\llmdist(\tokens^\ast|\sentence)\approx 0.474$; however $\tokens=\texttt{[\implws tongue,less]}$ is a more likely tokenization according to the LLM, with $\llmdist(\tokens|\sentence)\approx 0.518$.
This gap between the probability of the most likely tokenization and canonical tokenization can be even more extreme. For instance, $\tokens=\texttt{[Hyp,no,patu,rist]}$ is a much more likely tokenization $\llmdist(\tokens|\sentence)=0.9948$ compared to the canonical tokenization $\tokens^\ast=\texttt{[Hyp,nop,atu,rist]}$, with the latter taking only $\llmdist(\tokens^\ast|\sentence)=0.0004$ of the total mass.

\begin{figure}[t]
    \centering
    \input{tikz/canon_perc_length}
    \caption{\textbf{Canonicity in generated text.} Percentage of canonicity drops as more tokens are generated.}
    \label{fig:canon-perc}
\end{figure}
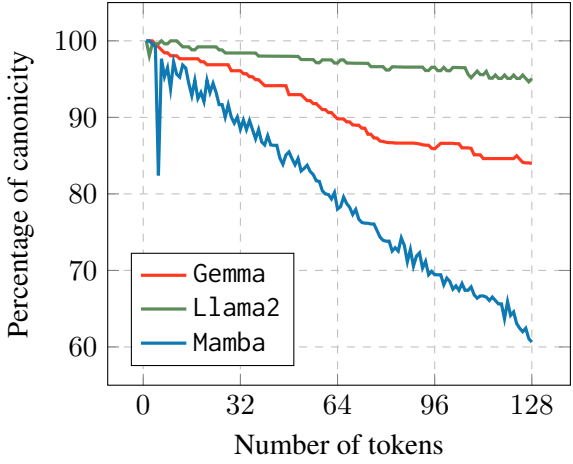

This seems to suggest that there is some mass being attributed to non-canonical tokenizations, especially over longer text.
We thus raise another question: instead of using a single tokenization, could we aggregate over all tokenizations, each weighted by their probability, effectively computing the marginal probability of a given string?

\section{... and Computing the Marginal Probability is Also Hard}

Evaluating the probability of a string requires marginalizing over all its possible tokenizations.
We now formally define this task and show it to be computationally hard.
\begin{restatable}[Marginal String Probability]{problem}{prbMarg}\label{prob: par}
    Let $\tokens$ denote a token sequence. Given a string $\sentence$ and an autoregressive LLM $\llmdist$, the marginal string probability problem is to compute 
    \begin{equation*}
        \llmdist(\sentence)=\sum_{\tokens}\llmdist\left(\tokens,\sentence\right).
    \end{equation*}
    
\end{restatable}
\begin{restatable}{theorem}{thmMarg} \label{thm:tothard}
    The marginal string probability problem is \#P-hard.
\end{restatable}
\begin{proof}
    (Sketch) The proof is by reduction from the counting version of the 3-SAT Boolean satisfiability problem (\#3-SAT), which is known to be \#P-complete. We encode the Boolean variables as possible tokens in a string, such that there is a correspondence between the number of satisfying assignments of the Boolean formula and the marginal probability of the string under the LLM. The full proof can be found in Appendix \ref{apxsec:hardness}.
\end{proof}

\subsection*{Marginal Probability Estimation}

\begin{figure*}[t]
    \centering
    \begin{subfigure}[t]{0.25\textwidth}
    \centering
    \input{tikz/mar_curve}
    \caption{String probability estimates}\label{fig:mar-curve}
    \end{subfigure}~\begin{subfigure}[t]{0.75\textwidth}
    \centering
    \input{tikz/convergence}
    \caption{Log probability difference between approximate marginal and canonical probability}\label{fig:conv-dist}
    \end{subfigure}
    \caption{\textbf{Convergence of approximate marginal.} 
    (a)~the approximate marginal string probability as a function of the number of samples (\ref*{pgf:marginal}), compared to the canonical probability (\ref*{pgf:canonical}) and the true marginal (\ref*{pgf:true-marginal}). 
    (b)~average absolute difference in log-likelihood between the approximate marginal and the canonical probability for different strings across \textcolor{palette-orange}{\textbf{\gemma{}}}, \textcolor{palette-green}{\textbf{\llama{}}} and \textcolor{palette-blue}{\textbf{\mamba{}}} in color, with individual examples in gray (\ref*{pgf:marginal-indiv}).} \label{fig:convergence}
\end{figure*}
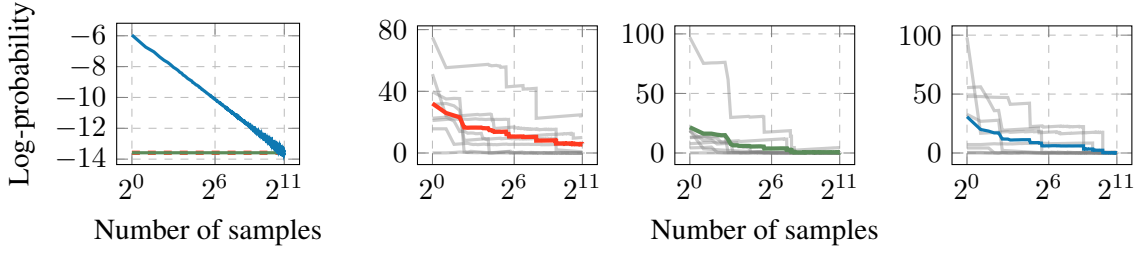

In light of the above hardness results, we now shift our attention to approximating the marginal string probability.
In particular, we will focus on estimators based on \emph{sequential importance sampling} \cite{kloek78,geweke89}.
In this instance of importance sampling, we sample tokenizations $\tokens$ given a string $\sentence$ according to some proposal distribution $\proposaldist(\tokens|\sentence)$. Given a set of samples $\tokens^{(1)}, \dots, \tokens^{(N)}$ from this distribution, an estimate of the marginal string probability $\llmdist(\sentence)$ is
\begin{align}\label{eq:mar-estimate}
    \llmdist(\sentence)\!=\!\mathbb{E}_{\tokens \sim q(\tokens|\sentence)} \left[ \frac{\llmdist(\sentence, \tokens)}{\proposaldist(\tokens|\sentence)} \right]\!\approx\!\frac{1}{N} \sum_{i=1}^{N} \frac{\llmdist(\sentence, \tokens^{(i)})}{\proposaldist(\tokens^{(i)}|\sentence)}.
\end{align}

A simple proposal distribution one might consider is the prior LLM token distribution:
\begin{equation*}
    \proposaldist_{\text{LLM}}(\tokens|\sentence) := \prod_{j=1}^{|\tokens|} \llmdist(\token_j|\tokens_{1:j-1}),
\end{equation*}
where $\llmdist(\token_j|\tokens_{1:j-1})$ is the LLM next-token distribution.
However, estimating $\llmdist(\sentence)$ this way requires rejecting all sampled token sequences where $\tokens\not\models\sentence$, making the approach infeasible in practice.

To address this issue, we use a modified proposal distribution: the \emph{1-step look-ahead 
proposal distribution}, first proposed in \citet{chirkova23}.
This distribution adjusts the LLM's next-token distribution at each step by checking whether the upcoming
token $v_j$, combined with the previous tokens, forms a tokenization of a prefix of the string~$\sentence$.
Intuitively, we iteratively prune away from the support of the distribution tokenizations not consistent
with $x$.
This can be done efficiently by simply traversing the MDD compiled from the string and masking out all tokens that are not compatible with the labels of the outgoing edges at the current node.
Formally, the proposal distribution is
\begin{align*}
    &\proposaldist_{\text{LA}}(\tokens|\sentence) \defeq\prod_{j=1}^{|\tokens|}\proposaldist_{\text{LA}}(\token_j|\tokens_{1:j-1},\sentence)\text{,~~ where}\\
    &\proposaldist_{\text{LA}}(\token_j|\tokens_{1:j-1},\sentence)\propto\llmdist(\token_j|\tokens_{1:j-1})\liv\tokens_{1:j}\models\sentence_{1:}\riv.
\end{align*}
Here, $\liv\tokens_{1:j}\models\sentence_{1:}\riv$ evaluates to $1$ if $\tokens_{1:j}$ forms a tokenization of a prefix of string $\sentence$, and $0$ otherwise.
Essentially, the proposal distribution is a \emph{greedy} and \emph{myopic} approximation of the LLM distribution over tokenizations at every step of the sequence. 

Interestingly, we observe that, for short strings where we are able to compute the true marginal, even though the proposal eventually converges to the true marginal as the number of samples increases, the probability of the canonical tokenization is just as close to the true marginal.
This seems to suggest that the canonical probability is, in fact, practically the marginal probability of the string in these cases.
\Cref{fig:mar-curve} shows one instance of the approximate marginal slowly converging to the true marginal as the number of samples increases for a single small example of the \textsc{OpenBookQA} dataset~\cite{openbookqa}.

\begin{figure*}[t]
    \centering
    \input{tikz/acc_curve}
    \caption{\textbf{Accuracy of approximated marginal string probability over number of samples.} Solid curves (\protect\tikz[baseline=-0.5ex]{\protect\draw[gray,very thick] (0,0) -- (0.5,0);}) show marginal mean accuracy, shaded areas (\protect\tikz[baseline=0.25ex]{\protect\draw[gray!70!white,fill=gray!40!white,fill opacity=0.3] (0,0) rectangle ++(0.5,0.25);}) show marginal standard deviation, dashed lines (\protect\tikz[baseline=-0.5ex]{\protect\draw[gray,very thick,dashed] (0,0) -- (0.5,0);}) show canonical baseline across \textcolor{palette-orange}{\textbf{\gemma{}}}, \textcolor{palette-green}{\textbf{\llama{}}} and \textcolor{palette-blue}{\textbf{\mamba{}}}.}\label{fig:acc_curve}
\end{figure*}

\begin{table*}[t]
    \centering
    \begin{tabular}{l|ccc|ccc|ccc|c}
        \hline
        \hline
                 & \multicolumn{3}{c|}{\textsc{HellaSwag}} & \multicolumn{3}{c|}{\textsc{SocialIQA}} & \multicolumn{3}{c|}{\textsc{OpenBookQA}} & \\
                 & \textsc{Can} & \textsc{Mar} & \textsc{Diff} & \textsc{Can} & \textsc{Mar} & \textsc{Diff} & \textsc{Can} & \textsc{Mar} & \textsc{Diff} & \textsc{Avg Diff}\\
        \hline
        \llama{} & 59.6 & \textbf{60.4} & 0.81 & 44.1 & \textbf{45.4} & 1.33 & 30.8 & \textbf{33.1} & 2.33 & 1.49 \\
        \gemma{} & 54.7 & \textbf{55.8} & 1.10 & 48.7 & \textbf{48.9} & 0.21 & 30.2 & \textbf{31.0} & 0.76 & 0.69 \\
        \mamba{} & \textbf{32.4} & 31.6 & -0.80 & 39.1 & \textbf{40.9} & 1.80 & 16.6 & \textbf{22.3} & 5.69 & 2.23 \\
        \hline
        \hline
    \end{tabular}
    \caption{\textbf{Accuracy after tuning the number of samples to estimate marginal string probability.} \textsc{Can} stands for the canonical baseline accuracy, \textsc{Mar} for the tuned approximate marginal, and \textsc{Diff} the difference between the last two. \textbf{Bold} entries indicate highest accuracy. The last column shows the average difference across the three datasets.}\label{tab:tuned-acc}
\end{table*}

For longer text, we observe that, for most cases, the approximate marginal also converges close to the canonical probability.
As the number of tokenizations to be summed out is enormous, we are unable to compute the true marginal.
In none of the cases we evaluated, the approximate marginal probability was meaningfully higher than the canonical.
\Cref{fig:conv-dist} shows the difference in log-probability of several marginal estimates across different architectures and \textsc{OpenBookQA} strings.
Notably, estimates that were very different from the canonical probability contained no canonical samples, further confirming that most of the probability mass is in the canonical tokenization.

So far, we have presented empirical evidence that seems to confirm that: (1) canonical is, in most cases, the most likely tokenization, and (2) it carries so much of the probability mass that it is practically the marginal itself.
Curiously, in an arguably contradictory twist, we experimentally show evidence that suggests that there exists some signal in non-canonical tokenizations to the point where we are able to achieve consistently better downstream performance in Q\&A tasks.

\section{Non-Canonical Tokenizations in Question Answering}

In multiple-choice question answering, a model is given a question (possibly with context) and is asked to choose between a number of different answers to the question. Typically, this is performed by evaluating the probability of each answer under the default canonical tokenization, and selecting the answer with the highest probability.
Formally, given a question $\question$ with canonical tokenization $\tokens_{\question}^\ast$ and set of $K$ answers $\{\answer_i\}_{i=1}^{K}$ with canonical tokenizations $\{\tokens_{\answer_i}^\ast\}_{i=1}^K$, the classification is given by
\begin{equation*}
    \arg \max_i \llmdist(\tokens_{\answer_i}^\ast|\tokens_{\question}^\ast).
\end{equation*}

Alternatively, we can compute these probabilities over other tokenizations, for instance, by computing an approximation to the marginal:
\begin{equation*}
\arg \max_i \llmdist(\answer_i|\tokens_{\question}^\ast)\!=\!\arg \max_i \sum_{\tokens_{\answer_i}\models \answer_i}\llmdist(\tokens_{\answer_i}|\tokens_{\question}^\ast).
\end{equation*}

From prior discussion, we expect the approximate marginal to gradually converge to canonical.
However, we empirically find that, before convergence, there is a surprising increase in accuracy when weighting over non-canonical tokenizations compared to the canonical baseline.
\Cref{fig:acc_curve} shows accuracy for the marginal approximation as a function of the number of samples in three different question answering datasets: \textsc{HellaSwag}~\cite{hellaswag}, \textsc{SocialIQA}~\cite{socialiqa} and \textsc{OpenBookQA}~\cite{openbookqa}.
Due to computational constraints, we only evaluate on randomly sampled subsets of 1000 examples for each dataset.%

\begin{figure*}
    \centering
    \input{tikz/mixture}
    \caption{\textbf{Accuracy for mixture of canonical and non-canonical tokenizations.} Solid curves (\protect\tikz[baseline=-0.5ex]{\protect\draw[palette-blue,very thick] (0,0) -- (0.5,0);}) show accuracy for the mixture of non-canonical and canonical tokenizations across models and datasets. Dashed lines (\protect\tikz[baseline=-0.5ex]{\protect\draw[palette-blue,very thick,dashed] (0,0) -- (0.55,0);}) show the canonical baseline.%
    }\label{fig:mix-classifier}
\end{figure*}
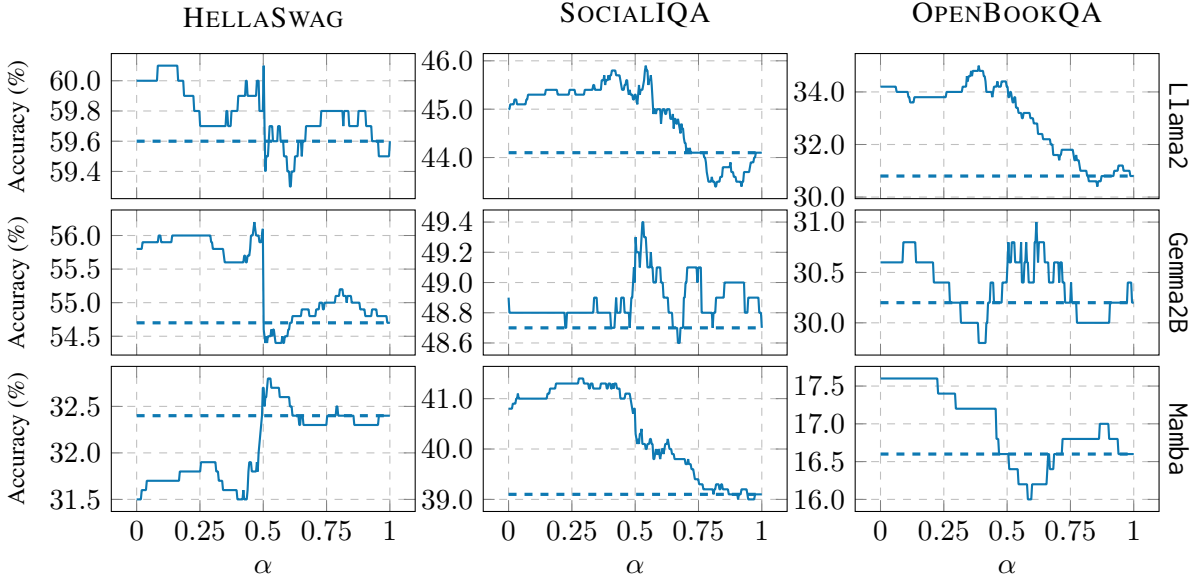

\begin{table*}[t]
    \centering
    \begin{tabular}{l|ccc|ccc|ccc}
        \hline
        \hline
                 & \multicolumn{3}{c|}{\textsc{HellaSwag}} & \multicolumn{3}{c|}{\textsc{SocialIQA}} & \multicolumn{3}{c}{\textsc{OpenBookQA}}\\
                 & \texttt{Llama3} & \texttt{Gemma} & \texttt{Mamba} & \texttt{Llama2} & \texttt{Gemma} & \texttt{Mamba} & \texttt{Llama2} & \texttt{Gemma} & \texttt{Mamba}\\
        \hline
        \textsc{Mixture} & \textbf{59.7} & \textbf{55.8} & 31.6 & \textbf{44.8} & \textbf{48.8} & \textbf{39.8} & \textbf{34.0} & \textbf{30.6} & \textbf{17.6}\\
        \textsc{Canonical} & 59.6 & 54.7 & \textbf{32.4} & 44.1 & 48.7 & 39.1 & 30.8 & 30.2 & 16.6\\
        \hline
        \hline
    \end{tabular}
    \caption{\textbf{Mixture accuracy after $\alpha$ tuning.} The first row shows accuracy values for the mixture, while the second row shows the canonical baseline. Entries in \textbf{bold} indicate highest accuracy.}\label{tab:mix-tuned-acc}
\end{table*}

By parameter tuning the number of samples, we are able to achieve a consistent performance increase in accuracy, as \Cref{tab:tuned-acc} shows.
Tuning the number of samples consisted of sampling 256 samples from a 1000 examples validation hold-out subset, computing the accuracy for 256 trials of 256-choose-$k$ samples, and taking the $k$ which maximizes average accuracy on the hold-out subset.
This $k$ is then used to sample that number of tokenizations in the test set.
The precise values of $k$ are shown in \Cref{tab:full-tuned-acc} in appendix.

To further understand how much non-canonical tokenizations play a role in this improvement, and find out where the signal in tokenization comes from, we construct a mixture of canonical and non-canonical tokenizations. We evaluate the improvement in accuracy, effectively measuring how much non-canonicity plays a role in the downstream task.
More formally, we compute
\begin{align*}
    \arg \max_i\enspace \alpha&\cdot\llmdist\!\left(\tokens_{\answer_i}^\ast|\tokens_{\question}^\ast,\tokens_{\answer_1}^\ast\vee \tokens_{\answer_2}^\ast\vee\ldots\vee\tokens_{\answer_k}^\ast\right)\,+\nonumber \\
    (1-\alpha)&\cdot\llmdist(\tilde{\tokens}_{\answer_i}|\tokens_{\question}^\ast,\tilde{\tokens}_{\answer_1}\vee \tilde{\tokens}_{\answer_2}\vee\ldots\vee \tilde{\tokens}_{\answer_k}),
\end{align*}
where $0\leq\alpha\leq 1$ and we use $\tilde{\tokens}_{\answer_i}$ to denote the set of all non-canonical tokenizations of $\answer_i$, i.e.\ $\tilde{\tokens}_{\answer_i}=\{\tokens:(\tokens\neq\tokens_{\answer_i}^\ast)\wedge(\tokens\models \answer_i)\}$.
When $\alpha=1$, the equation above reduces to the standard canonical tokenization baseline, while $\alpha=0$ weighs only non-canonical tokenizations of the same answer.
To make sure that both terms are on the same scale, we condition the distributions on the possible answers, yielding two classifiers over answers instead of tokenizations.

We approximate $\llmdist(\tilde{\tokens}_{\answer_i}|\tokens_{\question}^\ast,\tilde{\tokens}_{\answer_1}\vee\ldots\vee\tilde{\tokens}_{\answer_k})$ by computing the (unbiased) marginal estimate over tokenizations that are \emph{not} canonical, i.e.,
\begin{align*}\label{eq:noncanon-mar} \allowdisplaybreaks
    &\llmdist(\tilde{\tokens}_{\answer_i}|\tokens_{\question}^\ast,\tilde{\tokens}_{\answer_1}\vee\ldots\vee\tilde{\tokens}_{\answer_k}) \propto \llmdist(\tilde{\tokens}_{\answer_i}|\tokens_{\question}^\ast)\\
    &\quad =\mathop{\mathbb{E}}_{\tokens\sim\proposaldist(\tokens|\answer_i,\tokens_{\question}^\ast)}\left[\frac{\llmdist\!\left(\tokens,\answer_i|\tokens_{\question}^\ast\right)}{\proposaldist\!\left(\tokens|\answer_i,\tokens_{\question}^\ast\right)}\cdot\liv\tokens\neq \tokens_{\answer_i}^\ast\riv\right]\nonumber\\
    &\quad\approx\frac{1}{N}\sum_{j=1}^N\frac{\llmdist\!\left(\tokens^{(j)}|\tokens_{\question}^\ast\right)}{\proposaldist\!\left(\tokens^{(j)}|\answer_i,\tokens_{\question}^\ast\right)}\cdot\left\liv\tokens^{(j)}\neq \tokens_{\answer_i}^\ast\right\riv,\nonumber
\end{align*}
In practice, this amounts to zeroing-out all importance weights that \emph{are} canonical.
The first term of the mixture is computed by simply normalizing the standard canonical probability over the canonical tokenizations of possible answers
\begin{equation*}
    \llmdist(\tokens_{\answer_i}^\ast|\tokens_{\question}^\ast,\tokens_{\answer_1}^\ast\vee\ldots\vee\tokens_{\answer_k}^\ast)\propto\llmdist(\tokens_{\answer_i}^\ast|\tokens_{\question}^\ast).
\end{equation*}
The resulting mixture is then a weighted version of the marginal $\llmdist(\answer_i|\tokens_c^\ast,\answer_1\vee\ldots\vee \answer_k)$ where we adjust the mass attributed to canonical according to parameter $\alpha$.
This allows us to inspect how the model behaves when mass is ``shoveled'' around from non-canonical to canonical and vice versa.

\Cref{fig:mix-classifier} shows how this mixture behaves for different values of $\alpha$ downstream, while \Cref{tab:mix-tuned-acc} shows the performance change when tuning $\alpha$ in the validation set and applying it on the test set. Precise tuned $\alpha$ values are shown in \Cref{tab:full-mix-tuned-acc}.

These experiments show that there is clear and significant signal in non-canonical tokenizations to the point that we are able to achieve a consistent increase in accuracy, suggesting that non-canonical tokenizations do indeed retain meaningful information in LLMs, and hopefully motivating further research in this direction.

\section{Conclusion}

Modern language models make the assumption that text is represented by a unique canonical tokenization equivalent to the text itself.
We showed that not only is the space of possible tokenizations exponential, but prove that reasoning probabilistically about this space is hard.
We then showed empirical evidence that suggests that, in practice, not only is the canonical tokenization the most likely tokenization, but it is also very close to the probability of the text itself (i.e.\ the marginal probability of the text summed over tokenizations).
Despite this, we present surprising evidence of significant signal in non-canonical tokenizations, thus motivating further research on non-canonical tokenizations.

\section{Limitations}

Having to evaluate several tokenizations instead of just the canonical one is an obvious limitation to this work, as it requires many more (possibly costly) calls to the LLM.
Due to computational constraints, we were also unable to provide evaluation results with regards to larger LLMs.

\section*{Acknowledgements}

This work was funded in part by the DARPA ANSR program under award FA8750-23-2-0004, the DARPA PTG Program under award HR00112220005, and
NSF grant \#IIS-1943641.
This work was done in part while BW and GVdB were visiting the Simons Institute for the Theory of Computing.

\bibliography{example_paper}

\appendix

\input{appendix/problem}

\input{appendix/hardness_corrected}

\section{Experiments}

\begin{table}[ht]
    \centering
    \resizebox{\columnwidth}{!}{
    \begin{tabular}{l|ccc|c}
        \hline
        \hline
         & \multicolumn{1}{l}{\llama{}} & \multicolumn{1}{l}{\gemma{}} & \multicolumn{1}{l|}{\mamba{}} &\\
        \hline
         \textsc{Can}           & 59.6 & 54.7 & \textbf{32.4} &\parbox[t]{2mm}{\multirow{4}{*}{\rotatebox[origin=c]{270}{\scriptsize\strut\textsc{HellaSwag}}}}\\
         \textsc{Mar}           & \textbf{60.4} & \textbf{55.8} & 31.6 & \\
         \textsc{Diff}          & 0.81 & 1.10 & -0.80 & \\
         \textsc{\# Samples}    &   41 &  256 &  256 & \\
        \hline
         \textsc{Can}           & 44.1 & 48.7 & 39.1 & \parbox[t]{2mm}{\multirow{4}{*}{\rotatebox[origin=c]{270}{\scriptsize\strut\textsc{SocialIQA}}}}\\
         \textsc{Mar}           & \textbf{45.4} & \textbf{48.9} & \textbf{40.9} & \\
         \textsc{Diff}          & 1.33 & 0.21 & 1.80 & \\
         \textsc{\# Samples}    &  238 &  140 &  256 & \\
        \hline
         \textsc{Can}           & 30.8 & 30.2 & 16.6 & \parbox[t]{2mm}{\multirow{4}{*}{\rotatebox[origin=c]{270}{\scriptsize\strut\textsc{OpenBookQA}}}}\\
         \textsc{Mar}           & \textbf{33.1} & \textbf{31.0} & \textbf{22.3} & \\
         \textsc{Diff}          & 2.33 & 0.76 & 5.69 & \\
         \textsc{\# Samples}    &    7 &  163 &    1 & \\
        \hline
        \textsc{Avg Diff}       & 1.49 & 0.69 & 2.23 & \\
        \hline
        \hline
    \end{tabular}
    }
    
    \caption{\textbf{Accuracy after parameter tuning.} See \Cref{tab:tuned-acc} for more details. \textsc{\# Samples} shows the number of samples after tuning on the validation set.}\label{tab:full-tuned-acc}
\end{table}

\begin{table}[ht!]
    \centering
    \resizebox{0.9\columnwidth}{!}{
    \begin{tabular}{l|cc|c}
        \hline
        \hline
         & \multicolumn{1}{l}{\textsc{Mixture}} & \multicolumn{1}{l|}{\phantom{$\neg$}\textsc{Canonical}} &\\
        \hline
        & \multicolumn{2}{c|}{Accuracy (\%)} & \\
        \hline
        \llama{} & \textbf{59.7} & 59.6 &\parbox[t]{2mm}{\multirow{3}{*}{\rotatebox[origin=c]{270}{\tiny\strut\textsc{HellaSwag}}}}\\
        \gemma{} & \textbf{55.8} & 54.7 & \\
        \mamba{} & 31.6 & \textbf{32.4} & \\
        \hline
        \llama{} & \textbf{44.8} & 44.1 & \parbox[t]{2mm}{\multirow{3}{*}{\rotatebox[origin=c]{270}{\tiny\strut\textsc{SocialIQA}}}}\\
        \gemma{}  & \textbf{48.8} & 48.7 &  \\
        \mamba{}  & \textbf{39.8} & 39.1 &  \\
        \hline
        \llama{} & \textbf{34.0} & 30.8 & \parbox[t]{2mm}{\multirow{3}{*}{\rotatebox[origin=c]{270}{\tiny\strut\textsc{OpenBookQA}}}}\\
        \gemma{}  & \textbf{30.6} & 30.2 &  \\
        \mamba{}  & \textbf{17.6} & 16.6 &  \\
        \hline
        \hline
        & \multicolumn{2}{c|}{Fine-tuned $\alpha$ values} &\\
        \hline
        \llama{} & 0.3445 & --- & \parbox[t]{2mm}{\multirow{3}{*}{\rotatebox[origin=c]{270}{\tiny\strut\textsc{HellaSwag}}}}\\
        \gemma{}  & 0.6421 & --- & \\
        \mamba{}  & 0.5151 & --- & \\
        \hline
        \llama{} & 0.3244 & --- & \parbox[t]{2mm}{\multirow{3}{*}{\rotatebox[origin=c]{270}{\tiny\strut\textsc{SocialIQA}}}}\\
        \gemma{}  & 0.4816 & --- & \\
        \mamba{}  & 0.0000 & --- & \\
        \hline
        \llama{} & 0.0201 & --- & \parbox[t]{2mm}{\multirow{3}{*}{\rotatebox[origin=c]{270}{\tiny\strut\textsc{OpenBookQA}}}}\\
        \gemma{}  & 0.5753 & --- & \\
        \mamba{}  & 0.0000 & --- & \\                
        \hline
        \hline
    \end{tabular}
    }
    \caption{\textbf{Mixture accuracy after $\alpha$ tuning.} See \Cref{tab:mix-tuned-acc} for more details. The lower portion of the table shows $\alpha$ values after tuning on the validation set.}\label{tab:full-mix-tuned-acc}
\end{table}

\Cref{fig:acc-curve-val} shows marginal estimates across models and datasets on the validation set as a function of the number of samples. \Cref{fig:mix-classifier} shows mixture performance as a function of parameter $\alpha$ on the validation set.

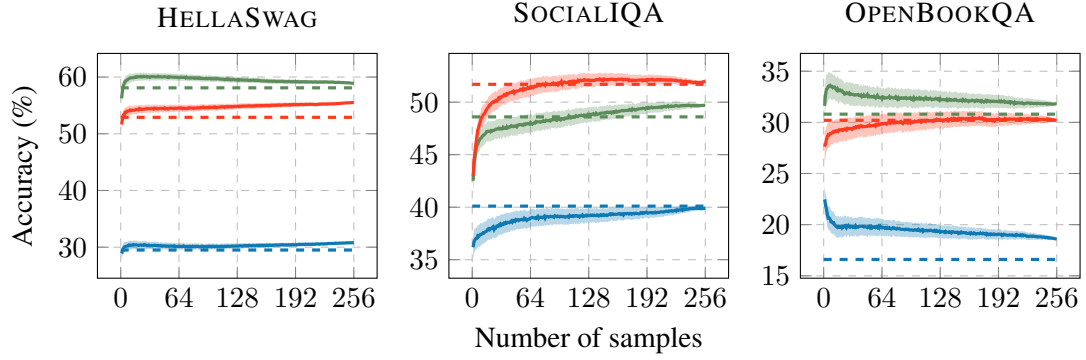
\begin{figure*}[t]
\centering
\pgfplotstableread[col sep=semicolon,trim cells]{
    model   ; dataset    ; accuracy
    Llama2  ; HellaSwag  ; 0.581
    Gemma2B ; HellaSwag  ; 0.529
    Mamba   ; HellaSwag  ; 0.295
    Llama2  ; SocialIQA  ; 0.486
    Gemma2B ; SocialIQA  ; 0.517
    Mamba   ; SocialIQA  ; 0.401
    Llama2  ; OpenBookQA ; 0.308
    Gemma2B ; OpenBookQA ; 0.302
    Mamba   ; OpenBookQA ; 0.166
}\canonicaldata

\def\datasizes{{1000,1000,1000}}
\def\names{{"SocialIQA", "OpenBookQA"}}
\def\tasks{ssocial_iqa_val, sopenbookqa_val}
\def\xaxislabels{{"Number of samples","\phantom{Number of samples}"}}
\begin{tikzpicture}
    \pgfmathsetmacro{\size}{1000}
    \begin{axis}[
        title={\textsc{HellaSwag}},
        xtick distance={64},
        ytick distance={0.1},
        yticklabel={\pgfmathparse{\tick*100}\pgfmathprintnumber{\pgfmathresult}},
        width=0.33\textwidth,
        xlabel={\phantom{Number of samples}},
        ylabel={Accuracy (\%)},
        xmajorgrids=true,
        ymajorgrids=true,
        xminorgrids=true,
        yminorgrids=true,
        grid=both,
        grid style=dashed,
        axis on top=false,
    ]
    \def\datapath{tikz/data/shellaswag_val_Llama2_\size_256_MarLM_llm.csv}
    \addplot[very thick,palette-green,select coords between index={0}{255}] table
        [x=x,y=mean,name path=mu,col sep=comma] {\datapath};
    \addplot[name path=std_high,draw=none,select coords between index={0}{255}] table
        [x=x,y expr=\thisrow{mean}+\thisrow{stdev},col sep=comma] {\datapath};
    \addplot[name path=std_low,draw=none,select coords between index={0}{255}] table
        [x=x,y expr=\thisrow{mean}-\thisrow{stdev},col sep=comma] {\datapath};
    \addplot[fill=palette-green,opacity=0.30] fill between [of=std_high and std_low];
    \pgfplotstablegetelem{0}{accuracy}\of\canonicaldata
    \pgfmathsetmacro\tmp{\pgfplotsretval}
    \addplot[mark=none,very thick,dashed,palette-green,domain=0:255] {\tmp};

    \def\datapath{tikz/data/shellaswag_val_Gemma2B_\size_256_MarLM_llm.csv}
    \addplot[very thick,palette-orange,select coords between index={0}{255}] table
        [x=x,y=mean,name path=mu,col sep=comma] {\datapath};
    \addplot[name path=std_high,draw=none,select coords between index={0}{255}] table
        [x=x,y expr=\thisrow{mean}+\thisrow{stdev},col sep=comma] {\datapath};
    \addplot[name path=std_low,draw=none,select coords between index={0}{255}] table
        [x=x,y expr=\thisrow{mean}-\thisrow{stdev},col sep=comma] {\datapath};
    \addplot[fill=palette-orange,opacity=0.30] fill between [of=std_high and std_low];
    \pgfplotstablegetelem{1}{accuracy}\of\canonicaldata
    \pgfmathsetmacro\tmp{\pgfplotsretval}
    \addplot[mark=none,very thick,dashed,palette-orange,domain=0:255] {\tmp};

    \def\datapath{tikz/data/shellaswag_val_Mamba_\size_256_MarLM_llm.csv}
    \addplot[very thick,palette-blue,select coords between index={0}{255}] table
        [x=x,y=mean,name path=mu,col sep=comma] {\datapath};
    \addplot[name path=std_high,draw=none,select coords between index={0}{255}] table
        [x=x,y expr=\thisrow{mean}+\thisrow{stdev},col sep=comma] {\datapath};
    \addplot[name path=std_low,draw=none,select coords between index={0}{255}] table
        [x=x,y expr=\thisrow{mean}-\thisrow{stdev},col sep=comma] {\datapath};
    \addplot[fill=palette-blue,opacity=0.30] fill between [of=std_high and std_low];
    \pgfplotstablegetelem{2}{accuracy}\of\canonicaldata
    \pgfmathsetmacro\tmp{\pgfplotsretval}
    \addplot[mark=none,very thick,dashed,palette-blue,domain=0:255] {\tmp};
    \end{axis}
\end{tikzpicture}
\foreach \d [count=\j] in \tasks {
    \begin{tikzpicture}
        \pgfmathsetmacro{\size}{\datasizes[\j-1]}
        \begin{axis}[
            title={\textsc{\pgfmathparse{\names[\j-1]}\pgfmathresult}},
            xtick distance={64},
            ytick distance={0.05},
            width=0.33\textwidth,
            yticklabel={\pgfmathparse{\tick*100}\pgfmathprintnumber{\pgfmathresult}},
            xlabel={\pgfmathparse{\xaxislabels[\j-1]}\pgfmathresult},
            xmajorgrids=true,
            ymajorgrids=true,
            xminorgrids=true,
            yminorgrids=true,
            grid=both,
            grid style=dashed,
            axis on top=false,
        ]
        \def\datapath{tikz/data/\d_Llama2_\size_256_MarLM_llm.csv}
        \addplot[very thick,palette-green,select coords between index={0}{255}] table
            [x=x,y=mean,name path=mu,col sep=comma] {\datapath};
        \addplot[name path=std_high,draw=none,select coords between index={0}{255}] table
            [x=x,y expr=\thisrow{mean}+\thisrow{stdev},col sep=comma] {\datapath};
        \addplot[name path=std_low,draw=none,select coords between index={0}{255}] table
            [x=x,y expr=\thisrow{mean}-\thisrow{stdev},col sep=comma] {\datapath};
        \addplot[fill=palette-green,opacity=0.30] fill between [of=std_high and std_low];
        \pgfmathtruncatemacro{\lpos}{(\j)*3}
        \pgfplotstablegetelem{\lpos}{accuracy}\of\canonicaldata
        \pgfmathsetmacro\tmp{\pgfplotsretval}
        \addplot[mark=none,very thick,dashed,palette-green,domain=0:255] {\tmp};

        \def\datapath{tikz/data/\d_Gemma2B_\size_256_MarLM_llm.csv}
        \addplot[very thick,palette-orange,select coords between index={0}{255}] table
            [x=x,y=mean,name path=mu,col sep=comma] {\datapath};
        \addplot[name path=std_high,draw=none,select coords between index={0}{255}] table
            [x=x,y expr=\thisrow{mean}+\thisrow{stdev},col sep=comma] {\datapath};
        \addplot[name path=std_low,draw=none,select coords between index={0}{255}] table
            [x=x,y expr=\thisrow{mean}-\thisrow{stdev},col sep=comma] {\datapath};
        \addplot[fill=palette-orange,opacity=0.30] fill between [of=std_high and std_low];
        \pgfmathtruncatemacro{\lpos}{1+(\j)*3}
        \pgfplotstablegetelem{\lpos}{accuracy}\of\canonicaldata
        \pgfmathsetmacro\tmp{\pgfplotsretval}
        \addplot[mark=none,very thick,dashed,palette-orange,domain=0:255] {\tmp};

        \def\datapath{tikz/data/\d_Mamba_\size_256_MarLM_llm.csv}
        \addplot[very thick,palette-blue,select coords between index={0}{255}] table
            [x=x,y=mean,name path=mu,col sep=comma] {\datapath};
        \addplot[name path=std_high,draw=none,select coords between index={0}{255}] table
            [x=x,y expr=\thisrow{mean}+\thisrow{stdev},col sep=comma] {\datapath};
        \addplot[name path=std_low,draw=none,select coords between index={0}{255}] table
            [x=x,y expr=\thisrow{mean}-\thisrow{stdev},col sep=comma] {\datapath};
        \addplot[fill=palette-blue,opacity=0.30] fill between [of=std_high and std_low];
        \pgfmathtruncatemacro{\lpos}{2+(\j)*3}
        \pgfplotstablegetelem{\lpos}{accuracy}\of\canonicaldata
        \pgfmathsetmacro\tmp{\pgfplotsretval}
        \addplot[mark=none,very thick,dashed,palette-blue,domain=0:255] {\tmp};
        \end{axis}
    \end{tikzpicture}
}
\caption{\textbf{Accuracy of approximated marginal over number of samples on the validation set.} Curves in solid \protect\tikz[baseline=-0.5ex]{\protect\draw[gray,very thick] (0,0) -- (0.5,0);} show marginal mean accuracy, shaded areas \protect\tikz{\protect\draw[gray!70!white,fill=gray!40!white,fill opacity=0.3] (0,0) rectangle ++(0.5,0.25);} show marginal standard deviation, dashed lines \protect\tikz[baseline=-0.5ex]{\protect\draw[gray,very thick,dashed] (0,0) -- (0.5,0);} show canonical baseline across \textcolor{palette-orange}{\textbf{\gemma{}}}, \textcolor{palette-green}{\textbf{\llama{}}} and \textcolor{palette-blue}{\textbf{\mamba{}}}.}\label{fig:acc-curve-val}
\end{figure*}

\begin{figure*}[t]
\pgfplotstableread[col sep=semicolon,trim cells]{
    model   ; dataset    ; accuracy
    Llama2  ; HellaSwag  ; 0.581
    Gemma2B ; HellaSwag  ; 0.529
    Mamba   ; HellaSwag  ; 0.295
    Llama2  ; SocialIQA  ; 0.486
    Gemma2B ; SocialIQA  ; 0.517
    Mamba   ; SocialIQA  ; 0.401
    Llama2  ; OpenBookQA ; 0.308
    Gemma2B ; OpenBookQA ; 0.302
    Mamba   ; OpenBookQA ; 0.166
}\canonicaldata
\pgfplotstableread[col sep=semicolon,trim cells]{
    model   ; dataset    ; mixture ; alpha
    Llama2  ; HellaSwag  ; marcan  ; 0.3445
    Gemma2B ; HellaSwag  ; marcan  ; 0.3244
    Mamba   ; HellaSwag  ; marcan  ; 0.0201
    Llama2  ; SocialIQA  ; marcan  ; 0.5151
    Gemma2B ; SocialIQA  ; marcan  ; 0.0000
    Mamba   ; SocialIQA  ; marcan  ; 0.0000
    Llama2  ; OpenBookQA ; marcan  ; 0.6421
    Gemma2B ; OpenBookQA ; marcan  ; 0.4816
    Mamba   ; OpenBookQA ; marcan  ; 0.5753
    Llama2  ; HellaSwag  ; ncancan ; 0.3445
    Gemma2B ; HellaSwag  ; ncancan ; 0.3244
    Mamba   ; HellaSwag  ; ncancan ; 0.0201
    Llama2  ; SocialIQA  ; ncancan ; 0.5151
    Gemma2B ; SocialIQA  ; ncancan ; 0.0000
    Mamba   ; SocialIQA  ; ncancan ; 0.0000
    Llama2  ; OpenBookQA ; ncancan ; 0.6421
    Gemma2B ; OpenBookQA ; ncancan ; 0.4816
    Mamba   ; OpenBookQA ; ncancan ; 0.5753
}\canonicaldata

\def\names{{"HellaSwag", "SocialIQA", "OpenBookQA"}}%
\def\tasks{{"shellaswag", "ssocial_iqa", "sopenbookqa"}}%
\def\models{{"Llama2", "Gemma2B", "Mamba"}}%
\begin{tikzpicture}
\begin{groupplot}[
    group style={
        group size=3 by 3,
        ylabels at=edge left,
        x descriptions at=edge bottom,
        horizontal sep=0.9cm,
        vertical sep=0.15cm,
    },
    y tick label style={
        /pgf/number format/fixed,
        /pgf/number format/fixed zerofill,
        /pgf/number format/precision=1,
        /pgf/number format/skip 0.,
    },
    yticklabel={\pgfmathparse{\tick*100}\pgfmathprintnumber{\pgfmathresult}},
    xtick distance=0.25,
    width=0.35\textwidth,
    height=3.5cm,
    xmajorgrids=true,
    ymajorgrids=true,
    xminorgrids=true,
    yminorgrids=true,
    xlabel={$\alpha$},
    ylabel={\small{}Accuracy (\%)},
    grid=both,
    grid style=dashed,
    axis on top=false,
    title style={font={\scshape}},
]
    \pgfplotsforeachungrouped \c in {0,...,8}{%
        \pgfmathparse{int(mod(\c,3))}
        \pgfmathsetmacro\j{\pgfmathresult}
        \pgfmathparse{int(\c/3)}
        \pgfmathsetmacro\i{\pgfmathresult}
        \pgfmathsetmacro\m{\models[\i]}
        \pgfmathsetmacro\t{\tasks[\j]}
        \pgfmathsetmacro\n{\names[\j]}
        \def\datapath{tikz/data/alpha_mixture_\m_\t_val.txt}
        \pgfmathtruncatemacro{\lpos}{\i+\j*3}
        \pgfplotstablegetelem{\lpos}{accuracy}\of\canonicaldata
        \pgfmathsetmacro\canonacc{\pgfplotsretval}
        \edef\tmp{%
            \noexpand\ifthenelse{\i=0}{
                \noexpand\nextgroupplot[title={\n}]
            }{\noexpand\nextgroupplot}
            \noexpand\addplot[thick,line join=bevel,palette-blue] table[x=i,y=ncan_can] {\datapath};
            \noexpand\addplot[mark=none,very thick,dashed,palette-blue,domain=0:1] {\canonacc};
        }\tmp%
    }
\end{groupplot}
\begin{groupplot}[
    group style={
        group size=3 by 3,
        ylabels at=edge right,
        horizontal sep=0.9cm,
        vertical sep=0.15cm,
    },
    ticks=none,
    axis line style=transparent,
    xmin=0, xmax=1,
    ymin=0, ymax=1,
    width=0.35\textwidth,
    height=3.5cm,
    axis on top=false,
    grid=none,
    ylabel style={rotate=180,font={\tt}},
]
    \pgfplotsforeachungrouped \c in {0,...,8}{%
        \pgfmathparse{int(mod(\c,3))}
        \pgfmathsetmacro\j{\pgfmathresult}
        \pgfmathparse{int(\c/3)}
        \pgfmathsetmacro\i{\pgfmathresult}
        \pgfmathsetmacro\m{\models[\i]}
        \edef\tmp{%
            \noexpand\ifthenelse{\j=2}{
                \noexpand\nextgroupplot[ylabel={\m}]
            }{\noexpand\nextgroupplot}
        }\tmp%
    }
\end{groupplot}
\end{tikzpicture}
\caption{\textbf{Accuracy for mixture of canonical and non-canonical tokenizations.} Solid curves (\protect\tikz[baseline=-0.5ex]{\protect\draw[palette-blue,very thick] (0,0) -- (0.5,0);}) show accuracy for the mixture of non-canonical and canonical tokenizations across models and datasets. Dashed lines (\protect\tikz[baseline=-0.5ex]{\protect\draw[palette-blue,very thick,dashed] (0,0) -- (0.55,0);}) show the canonical baseline.}
\end{figure*}
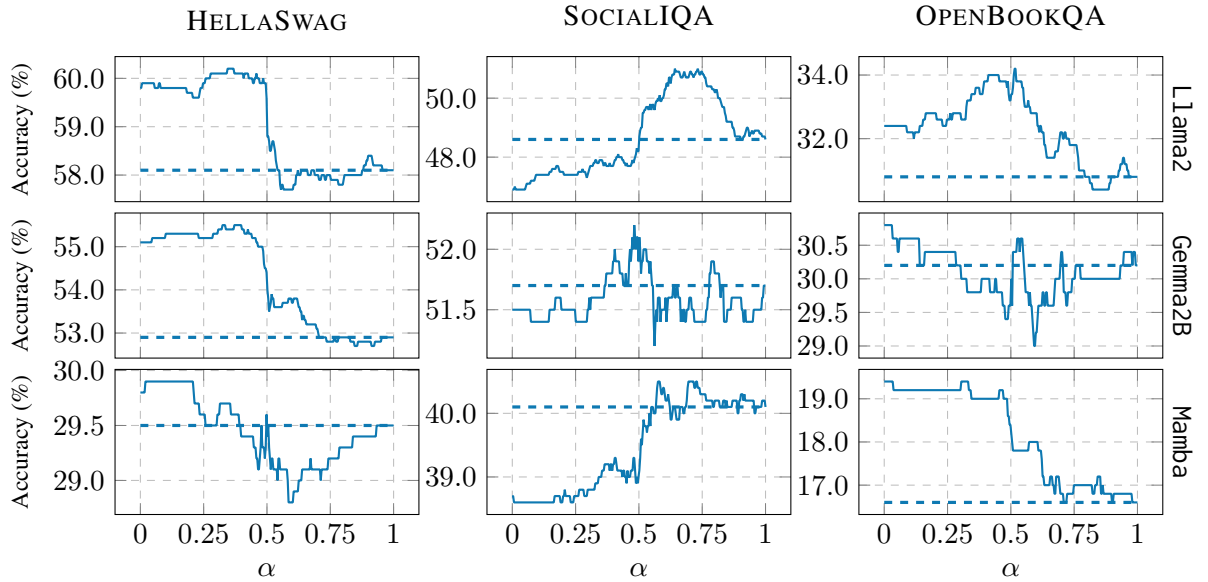

\end{document}

%% file: tikz_config.tex
\usepackage{tikz}
\usetikzlibrary{arrows,arrows.meta,bending,positioning,patterns,patterns.meta,calc}
\tikzset{tips=proper,edge/.style = {->,>=latex'},lbl/.style={draw=none,font={\footnotesize\ttfamily}}}
\usepackage{pgfplots}
\usepackage{pgfplotstable}
\pgfplotsset{compat=newest}
\usepgfplotslibrary{fillbetween,groupplots}
\newcommand{\implws}{\textvisiblespace}
\usepackage{pgffor}
\usepackage{ifthen}

\usepackage{multirow}

\definecolor{palette-orange}{HTML}{ff3a20}
\definecolor{palette-green}{HTML}{5b8c5a}
\definecolor{palette-blue}{HTML}{0e79b2}
\definecolor{palette-yellow}{HTML}{f5b700}
\definecolor{palette-dgreen}{HTML}{1e2f23}
\definecolor{palette-purple}{HTML}{331832}
\definecolor{dark gray}{HTML}{808080}
\definecolor{darker gray}{HTML}{606060}

\definecolor{palette-alt-green}{HTML}{084c61}
\definecolor{palette-alt-pink}{HTML}{d11149}
\definecolor{palette-alt-yellow}{HTML}{e3b505}
\definecolor{palette-alt-purple}{HTML}{331832}
\definecolor{palette-alt-blue}{HTML}{01baef}
\definecolor{palette-alt-gray}{HTML}{808f85}

\pgfplotscreateplotcyclelist{colorpalette}{
{palette-orange,fill=palette-orange},
{palette-green ,fill=palette-green },
{palette-blue  ,fill=palette-blue  },
{palette-purple,fill=palette-purple}}

\pgfplotsset{select coords between index/.style 2 args={
    x filter/.code={
        \ifnum\coordindex<#1\def\pgfmathresult{}\fi
        \ifnum\coordindex>#2\def\pgfmathresult{}\fi
    }
}}

%% file: macros.tex
\newcommand{\mpe}{\textnormal{MLT}}
\newcommand{\partition}{\textnormal{MSP}}

\newcommand{\llmclass}{\mathcal{L}}
\newcommand{\llmdist}{p}
\newcommand{\vocab}{\mathcal{V}}

\newcommand{\sentence}{\bm{x}}
\newcommand{\character}{x}

\newcommand{\numtokens}{m}

\newcommand{\token}{v}
\newcommand{\tokens}{\bm{v}}
\newcommand{\answer}{\bm{a}}
\newcommand{\question}{\bm{c}}

\newcommand{\variable}{a}
\newcommand{\literal}{l}
\newcommand{\literalidx}{I}
\newcommand{\literalsign}{P}
\newcommand{\clause}{c}
\newcommand{\cnf}{\psi}

\newcommand{\satclause}{S}

\newcommand{\proposaldist}{q}

\newcommand{\wrt}{w.r.t.\ }

\newcommand{\defeq}{\vcentcolon=}

\newcommand{\llama}{\texttt{Llama2}}
\newcommand{\gemma}{\texttt{Gemma}}
\newcommand{\mamba}{\texttt{Mamba}}

\makeatletter
\def\Ddots{\mathinner{\mkern1mu\raise\p@
\vbox{\kern7\p@\hbox{.}}\mkern2mu
\raise4\p@\hbox{.}\mkern2mu\raise7\p@\hbox{.}\mkern1mu}}
\makeatother

\newcommand{\mquote}[1]{\ensuremath{\text{``} #1 \text{''}}}

%% file: tikz/exp_growth.tex
\pgfplotstableread[col sep=semicolon,trim cells]{
    x  ; y            ; anc ; label
    13 ; 3632         ; 180 ; ``Tokenizations''
    18 ; 54480        ; 175 ; ``Tokenizations grow''
    26 ; 3759120      ; 175 ; ``Tokenizations grow rapidly''
    31 ; 48868560     ; 351 ; ``Tokenizations grow rapidly with''
    40 ; 8356523760   ; 0 ; ``Tokenizations grow rapidly with sentence''
    47 ; 350973997920 ; 0   ; ``Tokenizations grow rapidly with sentence length''
}\data
\begin{tikzpicture}
    \begin{axis}[
        ymode=log,
        ytick distance=100,
        width=\columnwidth,
        xlabel={Sentence length},
        ylabel={\# of tokenizations},
        xmin=5,
        xmajorgrids=true,
        ymajorgrids=true,
        grid style=dashed,
        width=\columnwidth,
    ]

    \addplot[
        color=palette-blue,
        mark=square,
        very thick,
        visualization depends on=\thisrow{anc} \as \anc,
        nodes near coords,
        point meta=explicit symbolic,
        every node near coord/.style={font=\tiny\sf\bfseries,anchor=\anc},
    ] table [
        meta index=3
    ] {\data};
    \end{axis}
\end{tikzpicture}

%% file: tikz/mdd.tex
\begin{tikzpicture}[every node/.style={draw,circle,inner sep=0pt,minimum size=12pt}]
        \node (n0) at (0, 0) {1};

        \node (n1) at ($(n0) + (0.75, -1.5)$) {2};
        \node (n2) at ($(n1) + (0.75, -1.0)$) {3};
        \node (n3) at ($(n0) + (0.0, -3.0)$) {4};
        \node (n4) at ($(n2) + (1.0, -1.5)$) {5};
        \node[rectangle] (n5) at ($(n3) + (-1.0, -1.5)$) {6};

        \draw[edge] (n0) to node[lbl,above right] {\implws} (n1);
        \draw[edge] (n0) to node[lbl,left,pos=0.5] {\implws\,Bi} (n3);
        \draw[edge] (n0) to ++(1.5, -0.5) -- node[lbl,pos=0.35,right] {\implws\,B} (n2);
        \draw[edge] (n1) to node[lbl,above left,pos=0.5] {Bi} (n3);
        \draw[edge] (n1) to node[lbl,above right] {B} (n2);
        \draw[edge] (n0) to ++(2.5, -0.5) -- node[lbl,pos=0.5,right] {\implws\,Bir} (n4);
        \draw[edge] (n4) to node[lbl,below right] {d} (n5);
        \draw[edge] (n0) to ++(-1.0, -0.5) -- node[lbl,pos=0.5,left] {\implws\,Bird} (n5);
        \draw[edge] (n2) to node[lbl,pos=0.5,above right] {ir} (n4);
        \draw[edge] (n2) to node[lbl,pos=0.5,above] {i} (n3);
        \draw[edge] (n3) to node[lbl,pos=0.4,above left] {rd} (n5);
        \draw[edge] (n3) to node[lbl,pos=0.65,above right] {r} (n4);
        \draw[edge] (n2) to ++(-0.5, -1.25) -- node[lbl,pos=0.15,above left] {\implws\,ird} (n5);
\end{tikzpicture}

%% file: tikz/branch_bound.tex
\begin{tikzpicture}
    \begin{axis}[
        ylabel={Time (in minutes)},
        xlabel={String length},
        ytick distance=20,
        width=\columnwidth,
        height=6.5cm,
        ymajorgrids=true, yminorgrids=true,
        grid=both,
        grid style=dashed,
        axis on top=false,
        legend style={
            at={(0.05, 0.50)},
            anchor=west,
        },
        legend cell align={left},
        ]
        \addplot[palette-orange,very thick] table[x=len,y=Gemma2B] {tikz/data/branch_bound_time.txt};
        \addplot[palette-green,very thick] table[x=len,y=Llama2] {tikz/data/branch_bound_time.txt};
        \addplot[palette-blue,very thick] table[x=len,y=Mamba] {tikz/data/branch_bound_time.txt};
        \legend{\gemma{},\llama{},\mamba{}}
        \addplot[gray,very thick,opacity=0.7,domain=0:102] {60} node[pos=0.2,below] {\sf\footnotesize\textbf{1-hour timeout}};
    \end{axis}
\end{tikzpicture}

%% file: tikz/hist_tokens.tex
\begin{tikzpicture}
    \pgfplotstableread[col sep=semicolon,trim cells]{
t ; logit ; prob
0 ; -2.7970573902130127 ; 0.06098926439881325
1 ; -8.397248268127441 ; 0.00022548694687429816
2 ; -11.685832023620605 ; 8.412162060267292e-06
3 ; -15.168355941772461 ; 2.5850368956525926e-07
4 ; -17.357582092285156 ; 2.8953285990951372e-08
5 ; -18.490615844726562 ; 9.324543803757024e-09
6 ; -24.11458969116211 ; 3.366408193872239e-11
7 ; -24.437774658203125 ; 2.436740566269524e-11
8 ; -25.024127960205078 ; 1.3556865699482223e-11
9 ; -25.91707992553711 ; 5.550795066644465e-12
10 ; -27.194114685058594 ; 1.547911213269082e-12
11 ; -27.211828231811523 ; 1.5207337342518223e-12
12 ; -27.27505874633789 ; 1.4275538216812489e-12
13 ; -27.46468734741211 ; 1.1809672163798357e-12
14 ; -30.039432525634766 ; 8.995808384228837e-14
15 ; -33.21503829956055 ; 3.757440695118918e-15
16 ; -35.72041320800781 ; 3.0677619584485556e-16
17 ; -37.20742416381836 ; 6.934593226466443e-17
18 ; -37.34132766723633 ; 6.065511614548419e-17
19 ; -38.132057189941406 ; 2.7507932629481563e-17
20 ; -38.230003356933594 ; 2.494137978590918e-17
21 ; -38.52330780029297 ; 1.860115950469843e-17
22 ; -39.226200103759766 ; 9.210384805537046e-18
23 ; -40.35053253173828 ; 2.992170769545724e-18
24 ; -40.64323806762695 ; 2.2328830737195174e-18
25 ; -42.74208450317383 ; 2.737465056834454e-19
26 ; -42.746192932128906 ; 2.726240991403934e-19
27 ; -43.409912109375 ; 1.4038305772896763e-19
28 ; -43.66197204589844 ; 1.091054588378682e-19
29 ; -43.90371322631836 ; 8.5676078512838e-20
30 ; -46.92692947387695 ; 4.1676715642692196e-21
31 ; -47.870452880859375 ; 1.6222824298980709e-21
32 ; -49.27819061279297 ; 3.96966419539738e-22
33 ; -50.832210540771484 ; 8.391729391430699e-23
34 ; -51.903564453125 ; 2.8745400344100697e-23
35 ; -53.01677703857422 ; 9.442919048706992e-24
36 ; -55.15247344970703 ; 1.1157968132112605e-24
37 ; -55.89659881591797 ; 5.3017048812898405e-25
38 ; -57.213157653808594 ; 1.4211547376900628e-25
39 ; -57.64808654785156 ; 9.199286038795274e-26
40 ; -58.048545837402344 ; 6.163634027856949e-26
41 ; -58.17741012573242 ; 5.41840884034872e-26
42 ; -58.79214096069336 ; 2.930202262028071e-26
43 ; -63.743560791015625 ; 2.0726371764184474e-28
44 ; -65.33651733398438 ; 4.21416041323463e-29
45 ; -66.45292663574219 ; 1.3799415945319914e-29
46 ; -68.26319122314453 ; 2.257733819567853e-30
47 ; -75.02571105957031 ; 2.610644199810687e-33
48 ; -76.16426086425781 ; 8.36145386622713e-34
49 ; -76.552490234375 ; 5.671211984805474e-34
50 ; -81.31753540039062 ; 4.833292835752499e-36
51 ; -92.64337158203125 ; 5.827720053434049e-41
    }\data
    \begin{axis}[
        area style,
        const plot,
        ymode=log,
        width=\columnwidth,
        height=6.5cm,
        ymax=1,ymin=1e-40,
        xmin=0,xmax=51,
        xlabel={Tokenizations},
        ylabel={Probability},
        every y tick label/.append style={font=\scriptsize,yshift=-0.5em,},
        xticklabels={},
        xmajorgrids=true,
        ymajorgrids=true,
        xminorgrids=true,
        yminorgrids=true,
        grid=both,
        grid style=dashed,
        enlargelimits=false,
        axis on top=false,
        legend style={
            at={(0.95, 0.95)},
            anchor=north east,
        },
        legend cell align={left},
    ]
    \addplot[forget plot,draw=none,mark=none,name path=zero] {1e-40};
    \addplot[forget plot,name path=pr] table [x=t,y=prob,palette-blue] from \data;
    \addplot[palette-blue!20!white] fill between[of=pr and zero];

    \addplot[forget plot,name path=can,draw=none,mark=none] table {
        x y
        0 0.0609892643988132
        1 1e-40
    };
    \addplot[palette-green!40!white,area legend] fill between[of=can and zero];

    \addplot[forget plot,name path=med,draw=none,mark=none] table {
        x y
        0 1e-40
        25 2.737465056834454e-19
        26 1e-40
    };
    \addplot[forget plot,pattern color=palette-orange!70!white,pattern={Lines[angle=-45, distance=0.75mm, line width=0.4mm]}] fill between[of=med and zero];

    \addplot[forget plot,name path=low,draw=none,mark=none] table {
        x y
        0 1e-40
        50 4.833292835752499e-36
        51 1e-40
    };
    \addplot[forget plot,pattern color=palette-purple!70!white,pattern={Lines[angle=-45, distance=0.75mm, line width=0.4mm]}] fill between[of=low and zero];
    \addplot[forget plot,draw=none,mark=none,domain=1:51,name path=poszero] {1e-40};
    \addplot[forget plot,name path=pr,area legend,draw=none,mark=none,name path=pospr,x filter/.expression={(x > 0 ? \pgfmathresult : NaN)}] table [x=t,y=prob,palette-blue] from \data;

    \legend{{\scriptsize$P(\neg\textup{Canonical}|\texttt{``Tokens''})\approx 0.004$},{\scriptsize$P(\textup{Canonical}|\texttt{``Tokens''})\phantom{\neg}\approx 0.996$}}

    \node[palette-green] (canon) at (11, 1e-20) {\small\texttt{[\implws Tok,ens]}};
    \draw[->,>=latex',thick,palette-green] (1.5, 1e-30) edge[bend right=25] (canon.south);
    \node[palette-purple,anchor=east] (char) at (50, 1e-15) {\small\texttt{[\implws,T,o,k,e,n,s]}};
    \draw[->,>=latex',thick,palette-purple] (50.5, 1e-35) edge[bend right=25] (char);
    \node[palette-orange,anchor=north] (mid) at (35, 1e-28) {\small\texttt{[\implws,To,ken,s]}};
    \draw[->,>=latex',thick,palette-orange] (26.5, 1e-38) edge[bend right=20] (mid.south);
    \end{axis}
\end{tikzpicture}

%% file: tikz/canon_perc_length.tex
\begin{tikzpicture}
    \begin{axis}[
        ylabel={Percentage of canonicity},
        xlabel={Number of tokens},
        width=\columnwidth,
        ymajorgrids=true, yminorgrids=true,
        xtick distance=32,
        grid=both,
        grid style=dashed,
        axis on top=false,
        ymin=55, ymax=105,
        legend style={
            at={(0.05, 0.05)},
            anchor=south west,
        },
        legend cell align={left},
        ]
        \addplot[palette-orange,very thick] table[x=i,y=Gemma2B] {tikz/data/canon_perc_len_data.txt};
        \addplot[palette-green,very thick] table[x=i,y=Llama2]  {tikz/data/canon_perc_len_data.txt};
        \addplot[palette-blue,very thick] table[x=i,y=Mamba]   {tikz/data/canon_perc_len_data.txt};
        \legend{\gemma{},\llama{},\mamba{}}
    \end{axis}
\end{tikzpicture}

%% file: tikz/mar_curve.tex
\begin{tikzpicture}[baseline]
    \begin{axis}[
        ylabel={Log-probability},
        xlabel={Number of samples},
        xmode=log,
        log basis x={2},        
        xtickten={0,6,11},
        width=\textwidth,
        ymajorgrids=true, yminorgrids=true,
        grid=both,
        grid style=dashed,
        axis on top=false,
        legend style={
            at={(0.95, 0.95)},
            anchor=north east,
            font={\footnotesize},
        },
        legend cell align={left},
        ]
        \addplot[palette-orange,ultra thick,dashed] table[x=i,y=lmar] {tikz/data/mar_curve_data.txt};\label{pgf:true-marginal}
        \addplot[palette-green,very thick] table[x=i,y=lcan] {tikz/data/mar_curve_data.txt};\label{pgf:canonical}
        \addplot[palette-blue,very thick] table[x=i,y=lmu] {tikz/data/mar_curve_data.txt};\label{pgf:marginal}
    \end{axis}
\end{tikzpicture}

%% file: tikz/convergence.tex
\begin{tikzpicture}[baseline]
    \begin{axis}[
        xmode=log,
        log basis x={2},
        xtickten={0,6,11},
        width=0.33\textwidth,
        ytick distance=40,
        enlarge x limits,
        ymajorgrids=true, yminorgrids=true,
        grid=both,
        grid style=dashed,
        axis on top=false,
        legend style={
            at={(0.95, 0.95)},
            anchor=north east,
            font={\footnotesize},
        },
        legend cell align={left},
        ]
        \addplot[gray,very thick,opacity=0.4] table[x=i,y=sopenbookqa_val_0] {tikz/data/Gemma2B_8_2048_data.txt};
        \addplot[gray,very thick,opacity=0.4] table[x=i,y=sopenbookqa_val_1] {tikz/data/Gemma2B_8_2048_data.txt};
        \addplot[gray,very thick,opacity=0.4] table[x=i,y=sopenbookqa_val_2] {tikz/data/Gemma2B_8_2048_data.txt};
        \addplot[gray,very thick,opacity=0.4] table[x=i,y=sopenbookqa_val_3] {tikz/data/Gemma2B_8_2048_data.txt};
        \addplot[gray,very thick,opacity=0.4] table[x=i,y=sopenbookqa_val_4] {tikz/data/Gemma2B_8_2048_data.txt};
        \addplot[gray,very thick,opacity=0.4] table[x=i,y=sopenbookqa_val_5] {tikz/data/Gemma2B_8_2048_data.txt};
        \addplot[gray,very thick,opacity=0.4] table[x=i,y=sopenbookqa_val_6] {tikz/data/Gemma2B_8_2048_data.txt};
        \addplot[gray,very thick,opacity=0.4] table[x=i,y=sopenbookqa_val_7] {tikz/data/Gemma2B_8_2048_data.txt};\label{pgf:marginal-indiv}
        \addplot[palette-orange,ultra thick] table[x=i,y=mu_sopenbookqa_val] {tikz/data/Gemma2B_8_2048_data.txt};
    \end{axis}
\end{tikzpicture}%
\begin{tikzpicture}[baseline]
    \begin{axis}[
        xlabel={Number of samples},
        xmode=log,
        log basis x={2},
        xtickten={0,6,11},
        width=0.33\textwidth,
        enlarge x limits,
        ymajorgrids=true, yminorgrids=true,
        grid=both,
        grid style=dashed,
        axis on top=false,
        legend style={
            at={(0.95, 0.95)},
            anchor=north east,
            font={\footnotesize},
        },
        legend cell align={left},
        ]
        \addplot[gray,very thick,opacity=0.4] table[x=i,y=sopenbookqa_val_0] {tikz/data/Llama2_8_2048_data.txt};
        \addplot[gray,very thick,opacity=0.4] table[x=i,y=sopenbookqa_val_1] {tikz/data/Llama2_8_2048_data.txt};
        \addplot[gray,very thick,opacity=0.4] table[x=i,y=sopenbookqa_val_2] {tikz/data/Llama2_8_2048_data.txt};
        \addplot[gray,very thick,opacity=0.4] table[x=i,y=sopenbookqa_val_3] {tikz/data/Llama2_8_2048_data.txt};
        \addplot[gray,very thick,opacity=0.4] table[x=i,y=sopenbookqa_val_4] {tikz/data/Llama2_8_2048_data.txt};
        \addplot[gray,very thick,opacity=0.4] table[x=i,y=sopenbookqa_val_5] {tikz/data/Llama2_8_2048_data.txt};
        \addplot[gray,very thick,opacity=0.4] table[x=i,y=sopenbookqa_val_6] {tikz/data/Llama2_8_2048_data.txt};
        \addplot[gray,very thick,opacity=0.4] table[x=i,y=sopenbookqa_val_7] {tikz/data/Llama2_8_2048_data.txt};
        \addplot[palette-green,ultra thick] table[x=i,y=mu_sopenbookqa_val] {tikz/data/Llama2_8_2048_data.txt};
    \end{axis}
\end{tikzpicture}%
\begin{tikzpicture}[baseline]
    \begin{axis}[
        xmode=log,
        log basis x={2},
        xtickten={0,6,11},
        width=0.33\textwidth,
        enlarge x limits,
        ymajorgrids=true, yminorgrids=true,
        grid=both,
        grid style=dashed,
        axis on top=false,
        legend style={
            at={(0.95, 0.95)},
            anchor=north east,
            font={\footnotesize},
        },
        legend cell align={left},
        ]
        \addplot[gray,very thick,opacity=0.4] table[x=i,y=sopenbookqa_val_0] {tikz/data/Mamba_8_4096_data.txt};
        \addplot[gray,very thick,opacity=0.4] table[x=i,y=sopenbookqa_val_1] {tikz/data/Mamba_8_4096_data.txt};
        \addplot[gray,very thick,opacity=0.4] table[x=i,y=sopenbookqa_val_2] {tikz/data/Mamba_8_4096_data.txt};
        \addplot[gray,very thick,opacity=0.4] table[x=i,y=sopenbookqa_val_3] {tikz/data/Mamba_8_4096_data.txt};
        \addplot[gray,very thick,opacity=0.4] table[x=i,y=sopenbookqa_val_4] {tikz/data/Mamba_8_4096_data.txt};
        \addplot[gray,very thick,opacity=0.4] table[x=i,y=sopenbookqa_val_5] {tikz/data/Mamba_8_4096_data.txt};
        \addplot[gray,very thick,opacity=0.4] table[x=i,y=sopenbookqa_val_6] {tikz/data/Mamba_8_4096_data.txt};
        \addplot[gray,very thick,opacity=0.4] table[x=i,y=sopenbookqa_val_7] {tikz/data/Mamba_8_4096_data.txt};
        \addplot[palette-blue,very thick] table[x=i,y=mu_sopenbookqa_val] {tikz/data/Mamba_8_4096_data.txt};
    \end{axis}
\end{tikzpicture}

%% file: tikz/acc_curve.tex
\pgfplotstableread[col sep=semicolon,trim cells]{
    model   ; dataset    ; accuracy
    Llama2  ; HellaSwag  ; 0.596
    Gemma2B ; HellaSwag  ; 0.547
    Mamba   ; HellaSwag  ; 0.324
    Llama2  ; SocialIQA  ; 0.441
    Gemma2B ; SocialIQA  ; 0.487
    Mamba   ; SocialIQA  ; 0.391
    Llama2  ; OpenBookQA ; 0.308
    Gemma2B ; OpenBookQA ; 0.302
    Mamba   ; OpenBookQA ; 0.166
}\canonicaldata

\def\datasizes{{1000,1000,1000}}
\def\names{{"SocialIQA", "OpenBookQA"}}
\def\tasks{ssocial_iqa, sopenbookqa}
\def\xaxislabels{{"Number of samples","\phantom{Number of samples}"}}
\begin{tikzpicture}
    \pgfmathsetmacro{\size}{1000}
    \begin{axis}[
        title={\textsc{HellaSwag}},
        xtick distance={64},
        ytick distance={0.1},
        yticklabel={\pgfmathparse{\tick*100}\pgfmathprintnumber{\pgfmathresult}},
        width=0.33\textwidth,
        xlabel={\phantom{Number of samples}},
        ylabel={Accuracy (\%)},
        xmajorgrids=true,
        ymajorgrids=true,
        xminorgrids=true,
        yminorgrids=true,
        grid=both,
        grid style=dashed,
        axis on top=false,
    ]
    \def\datapath{tikz/data/shellaswag_Llama2_\size_256_MarLM_llm.csv}
    \addplot[very thick,palette-green,select coords between index={0}{255}] table
        [x=x,y=mean,name path=mu,col sep=comma] {\datapath};
    \addplot[name path=std_high,draw=none,select coords between index={0}{255}] table
        [x=x,y expr=\thisrow{mean}+\thisrow{stdev},col sep=comma] {\datapath};
    \addplot[name path=std_low,draw=none,select coords between index={0}{255}] table
        [x=x,y expr=\thisrow{mean}-\thisrow{stdev},col sep=comma] {\datapath};
    \addplot[fill=palette-green,opacity=0.30] fill between [of=std_high and std_low];
    \pgfplotstablegetelem{0}{accuracy}\of\canonicaldata
    \pgfmathsetmacro\tmp{\pgfplotsretval}
    \addplot[mark=none,very thick,dashed,palette-green,domain=0:255] {\tmp};

    \def\datapath{tikz/data/shellaswag_Gemma2B_\size_256_MarLM_llm.csv}
    \addplot[very thick,palette-orange,select coords between index={0}{255}] table
        [x=x,y=mean,name path=mu,col sep=comma] {\datapath};
    \addplot[name path=std_high,draw=none,select coords between index={0}{255}] table
        [x=x,y expr=\thisrow{mean}+\thisrow{stdev},col sep=comma] {\datapath};
    \addplot[name path=std_low,draw=none,select coords between index={0}{255}] table
        [x=x,y expr=\thisrow{mean}-\thisrow{stdev},col sep=comma] {\datapath};
    \addplot[fill=palette-orange,opacity=0.30] fill between [of=std_high and std_low];
    \pgfplotstablegetelem{1}{accuracy}\of\canonicaldata
    \pgfmathsetmacro\tmp{\pgfplotsretval}
    \addplot[mark=none,very thick,dashed,palette-orange,domain=0:255] {\tmp};

    \def\datapath{tikz/data/shellaswag_Mamba_\size_256_MarLM_llm.csv}
    \addplot[very thick,palette-blue,select coords between index={0}{255}] table
        [x=x,y=mean,name path=mu,col sep=comma] {\datapath};
    \addplot[name path=std_high,draw=none,select coords between index={0}{255}] table
        [x=x,y expr=\thisrow{mean}+\thisrow{stdev},col sep=comma] {\datapath};
    \addplot[name path=std_low,draw=none,select coords between index={0}{255}] table
        [x=x,y expr=\thisrow{mean}-\thisrow{stdev},col sep=comma] {\datapath};
    \addplot[fill=palette-blue,opacity=0.30] fill between [of=std_high and std_low];
    \pgfplotstablegetelem{2}{accuracy}\of\canonicaldata
    \pgfmathsetmacro\tmp{\pgfplotsretval}
    \addplot[mark=none,very thick,dashed,palette-blue,domain=0:255] {\tmp};
    \end{axis}
\end{tikzpicture}
\foreach \d [count=\j] in \tasks {
    \begin{tikzpicture}
        \pgfmathsetmacro{\size}{\datasizes[\j-1]}
        \begin{axis}[
            title={\textsc{\pgfmathparse{\names[\j-1]}\pgfmathresult}},
            xtick distance={64},
            ytick distance={0.05},
            width=0.33\textwidth,
            yticklabel={\pgfmathparse{\tick*100}\pgfmathprintnumber{\pgfmathresult}},
            xlabel={\pgfmathparse{\xaxislabels[\j-1]}\pgfmathresult},
            xmajorgrids=true,
            ymajorgrids=true,
            xminorgrids=true,
            yminorgrids=true,
            grid=both,
            grid style=dashed,
            axis on top=false,
        ]
        \def\datapath{tikz/data/\d_Llama2_\size_256_MarLM_llm.csv}
        \addplot[very thick,palette-green,select coords between index={0}{255}] table
            [x=x,y=mean,name path=mu,col sep=comma] {\datapath};
        \addplot[name path=std_high,draw=none,select coords between index={0}{255}] table
            [x=x,y expr=\thisrow{mean}+\thisrow{stdev},col sep=comma] {\datapath};
        \addplot[name path=std_low,draw=none,select coords between index={0}{255}] table
            [x=x,y expr=\thisrow{mean}-\thisrow{stdev},col sep=comma] {\datapath};
        \addplot[fill=palette-green,opacity=0.30] fill between [of=std_high and std_low];
        \pgfmathtruncatemacro{\lpos}{(\j)*3}
        \pgfplotstablegetelem{\lpos}{accuracy}\of\canonicaldata
        \pgfmathsetmacro\tmp{\pgfplotsretval}
        \addplot[mark=none,very thick,dashed,palette-green,domain=0:255] {\tmp};

        \def\datapath{tikz/data/\d_Gemma2B_\size_256_MarLM_llm.csv}
        \addplot[very thick,palette-orange,select coords between index={0}{255}] table
            [x=x,y=mean,name path=mu,col sep=comma] {\datapath};
        \addplot[name path=std_high,draw=none,select coords between index={0}{255}] table
            [x=x,y expr=\thisrow{mean}+\thisrow{stdev},col sep=comma] {\datapath};
        \addplot[name path=std_low,draw=none,select coords between index={0}{255}] table
            [x=x,y expr=\thisrow{mean}-\thisrow{stdev},col sep=comma] {\datapath};
        \addplot[fill=palette-orange,opacity=0.30] fill between [of=std_high and std_low];
        \pgfmathtruncatemacro{\lpos}{1+(\j)*3}
        \pgfplotstablegetelem{\lpos}{accuracy}\of\canonicaldata
        \pgfmathsetmacro\tmp{\pgfplotsretval}
        \addplot[mark=none,very thick,dashed,palette-orange,domain=0:255] {\tmp};

        \def\datapath{tikz/data/\d_Mamba_\size_256_MarLM_llm.csv}
        \addplot[very thick,palette-blue,select coords between index={0}{255}] table
            [x=x,y=mean,name path=mu,col sep=comma] {\datapath};
        \addplot[name path=std_high,draw=none,select coords between index={0}{255}] table
            [x=x,y expr=\thisrow{mean}+\thisrow{stdev},col sep=comma] {\datapath};
        \addplot[name path=std_low,draw=none,select coords between index={0}{255}] table
            [x=x,y expr=\thisrow{mean}-\thisrow{stdev},col sep=comma] {\datapath};
        \addplot[fill=palette-blue,opacity=0.30] fill between [of=std_high and std_low];
        \pgfmathtruncatemacro{\lpos}{2+(\j)*3}
        \pgfplotstablegetelem{\lpos}{accuracy}\of\canonicaldata
        \pgfmathsetmacro\tmp{\pgfplotsretval}
        \addplot[mark=none,very thick,dashed,palette-blue,domain=0:255] {\tmp};
        \end{axis}
    \end{tikzpicture}
}

%% file: tikz/mixture.tex
\pgfplotstableread[col sep=semicolon,trim cells]{
    model   ; dataset    ; accuracy
    Llama2  ; HellaSwag  ; 0.596
    Gemma2B ; HellaSwag  ; 0.547
    Mamba   ; HellaSwag  ; 0.324
    Llama2  ; SocialIQA  ; 0.441
    Gemma2B ; SocialIQA  ; 0.487
    Mamba   ; SocialIQA  ; 0.391
    Llama2  ; OpenBookQA ; 0.308
    Gemma2B ; OpenBookQA ; 0.302
    Mamba   ; OpenBookQA ; 0.166
}\canonicaldata

\def\names{{"HellaSwag", "SocialIQA", "OpenBookQA"}}%
\def\tasks{{"shellaswag", "ssocial_iqa", "sopenbookqa"}}%
\def\models{{"Llama2", "Gemma2B", "Mamba"}}%
\begin{tikzpicture}
\begin{groupplot}[
    group style={
        group size=3 by 3,
        ylabels at=edge left,
        x descriptions at=edge bottom,
        horizontal sep=0.9cm,
        vertical sep=0.15cm,
    },
    y tick label style={
        /pgf/number format/fixed,
        /pgf/number format/fixed zerofill,
        /pgf/number format/precision=1,
        /pgf/number format/skip 0.,
    },
    yticklabel={\pgfmathparse{\tick*100}\pgfmathprintnumber{\pgfmathresult}},
    xtick distance=0.25,
    width=0.35\textwidth,
    height=3.5cm,
    xmajorgrids=true,
    ymajorgrids=true,
    xminorgrids=true,
    yminorgrids=true,
    xlabel={$\alpha$},
    ylabel={\small{}Accuracy (\%)},
    grid=both,
    grid style=dashed,
    axis on top=false,
    title style={font={\scshape}},
]
    \pgfplotsforeachungrouped \c in {0,...,8}{%
        \pgfmathparse{int(mod(\c,3))}
        \pgfmathsetmacro\j{\pgfmathresult}
        \pgfmathparse{int(\c/3)}
        \pgfmathsetmacro\i{\pgfmathresult}
        \pgfmathsetmacro\m{\models[\i]}
        \pgfmathsetmacro\t{\tasks[\j]}
        \pgfmathsetmacro\n{\names[\j]}
        \def\datapath{tikz/data/alpha_mixture_\m_\t_test.txt}
        \pgfmathtruncatemacro{\lpos}{\i+\j*3}
        \pgfplotstablegetelem{\lpos}{accuracy}\of\canonicaldata
        \pgfmathsetmacro\canonacc{\pgfplotsretval}
        \edef\tmp{%
            \noexpand\ifthenelse{\i=0}{
                \noexpand\nextgroupplot[title={\n}]
            }{\noexpand\nextgroupplot}
            \noexpand\addplot[thick,line join=bevel,palette-blue] table[x=i,y=ncan_can] {\datapath};
            \noexpand\addplot[mark=none,very thick,dashed,palette-blue,domain=0:1] {\canonacc};
        }\tmp%
    }
\end{groupplot}
\begin{groupplot}[
    group style={
        group size=3 by 3,
        ylabels at=edge right,
        horizontal sep=0.9cm,
        vertical sep=0.15cm,
    },
    ticks=none,
    axis line style=transparent,
    xmin=0, xmax=1,
    ymin=0, ymax=1,
    width=0.35\textwidth,
    height=3.5cm,
    axis on top=false,
    grid=none,
    ylabel style={rotate=180,font={\tt}},
]
    \pgfplotsforeachungrouped \c in {0,...,8}{%
        \pgfmathparse{int(mod(\c,3))}
        \pgfmathsetmacro\j{\pgfmathresult}
        \pgfmathparse{int(\c/3)}
        \pgfmathsetmacro\i{\pgfmathresult}
        \pgfmathsetmacro\m{\models[\i]}
        \edef\tmp{%
            \noexpand\ifthenelse{\j=2}{
                \noexpand\nextgroupplot[ylabel={\m}]
            }{\noexpand\nextgroupplot}
        }\tmp%
    }
\end{groupplot}
\end{tikzpicture}

%% file: appendix/problem.tex
\section{Problems}

For the purposes of studying the complexity of inference problems on induced tokenization distributions, we use $\llmclass$ to denote a class (set) of autoregressive large language models, and make the assumption that this set covers all possible autoregressive distributions:

\begin{assumption}[Expressivity of LLMs] \label{ass: expressivity}
    We assume that $\llmclass$ is sufficiently expressive: given any token sequence $\tokens = (\token_1, ..., \token_\numtokens)$,
    and sequence $\delta_1, ..., \delta_\numtokens$ with $\delta_i \in (0, 1)$ for all $i$, there exists $\llmdist \in \llmclass$ such that $\llmdist(\token_i | \token_1, ..., \token_{i-1}) = \delta_{i}$ for all $i = 1, ..., m$. %
\end{assumption}

Note that we do not require that the conditional probability take the value $0$ or $1$, as this cannot be expressed using logits. We also need to make the (reasonable) assumption that the conditional probability distribution of LLMs can be computed in polynomial time: 

\begin{assumption}[Complexity of LLMs] \label{ass: llmtime}
    We assume that for any $\llmdist \in \llmclass$, and any sequence of tokens $\tokens=(\token_1,\dots, \token_\numtokens)$, we can compute the distribution $\llmdist(\token_i | \token_1, ..., \token_{i-1})$ for any $i=1, \dots, m$ in polynomial time in $|\tokens|$.
\end{assumption}

Now, we consider two inference problems related to the induced tokenization distribution; namely, computing the most likely tokenization, and marginal string probability (a more formal statement of Problems \ref{prob: mpe} and \ref{prob: par}):

 \begin{problem}[Most Likely Tokenization] \label{prob: mpe_apx}
   Given a string $\sentence$, vocabulary $\vocab$, and an autoregressive LLM $\llmdist \in \llmclass$ over $\vocab$, and a threshold $\epsilon > 0$, we define the most likely tokenization problem $\mpe(\sentence, \vocab, \llmdist, \epsilon)$ as deciding whether:
   \begin{equation}
       \max_{\tokens} \llmdist(\tokens, \sentence) > \epsilon
   \end{equation}
\end{problem}

\begin{problem}[Marginal String Probability] \label{prob: par_apx}
   Given a string $\sentence$, vocabulary $\vocab$, and an autoregressive LLM $\llmdist \in \llmclass$ over $\vocab$, the marginal string probability problem $\partition(\sentence, \vocab, \llmdist)$ is to compute
   \begin{equation}
       \sum_{\tokens} \llmdist(\tokens, \sentence)
   \end{equation}
\end{problem}

%% file: appendix/hardness_corrected.tex
\section{Hardness} \label{apxsec:hardness}

In this section, we show that Problems \ref{prob: mpe_apx} and \ref{prob: par_apx} are both NP-hard. This will be achieved using a reduction from 3-SAT:

\begin{definition}[3-SAT]
    Given a set of Boolean variables $\variable_1, ..., \variable_n$, a Boolean formula is in 3-CNF if it is of the form:
    \begin{equation}
        \bigwedge_{k = 1}^{K} \clause_k
    \end{equation}
    where each clause $\clause_k$ is a disjunction of at most 3 literals (a literal is either a variable or its negation). 
    The 3-SAT problem is that of determining if a given 3-CNF formula is satisfiable.
\end{definition}

\thmMPE*

\begin{proof}
    We begin by showing hardness. Given an instance of 3-SAT, we construct an instance of \mpe{} such that the 3-CNF is satisfiable iff the maximal probability is above a certain threshold.  \\
    
 Let $\cnf = \bigwedge_{k=1}^{K} \clause_i$ be a 3-CNF formula consisting of $K$ clauses over $n$ Boolean variables, where $\clause_k = \literal_{k, 1} \vee \literal_{k, 2} \vee \literal_{k,3}$, and $\literal_{k, j}$ is a literal. For convenience, we write $\literalidx_{k, j}$ for the index of the variable $\literal_{k, j}$ refers to, and $\literalsign_{k, j}$ to be a Boolean variable which is true iff it is a positive literal, i.e. $\literal_{k, j} := (\variable_{\literalidx_{k, j}} = \literalsign_{k, j})$.  \\

 We now define a string $\sentence$ of length $3n + K$:
\begin{center}
abcabcabcabc...ddd...
\end{center}
where $\character_{3i + 1} = \mquote{a}, \character_{3i + 2} = \mquote{b}, \character_{3i + 1} = \mquote{c}$ for $0 \leq i \leq n-1$, and $\character_{i} = \mquote{d}$ for $3n+1 \leq i \leq 3n+K$.
We also define a vocabulary $\vocab = \{``a", ``bc", ``ab", ``c", ``d"\}$. \\

Finally, we define the LLM conditional probability distribution as follows. 
\begin{multline*}
	\llmdist(\token_{i}|\token_{0}, ..., \token_{i-1}) = \\
 \begin{cases}
		 0.45 & \text{if } i = 0 \wedge (\token_i = ``a"  \vee \token_i = ``ab")\\
          0.033 & \text{if } i = 0 \wedge \neg (\token_i = ``a"  \vee \token_i = ``ab")\\
		 0.9 & \text{if } i < 2n \wedge \token_{i-1} = ``a" \wedge \token_i = ``bc" \\
          0.025 & \text{if } i < 2n \wedge \token_{i-1} = ``a" \wedge \token_i \neq ``bc" \\
		 0.9 & \text{if } i < 2n \wedge \token_{i-1} = ``ab" \wedge \token_i = ``c" \\
          0.025 & \text{if } i < 2n \wedge \token_{i-1} = ``ab" \wedge \token_i \neq ``c" \\
		 0.45 & \text{if } i < 2n \wedge (\token_{i-1} = ``bc" \vee \token_{i-1} = ``c")  \\
        &\wedge (\token_i = ``a" \vee \token_i = ``ab") \\
          0.033 & \text{if } i < 2n \wedge (\token_{i-1} = ``bc" \vee \token_{i-1} =``c") \\
          &\wedge \neg (\token_i = ``a" \vee \token_i = ``ab") \\
          0.2 & \text{if } i < 2n \wedge (\token_{i-1} = ``d") \\
          0.9 & \text{if } i \geq 2n \wedge (\token_{i} = ``d") \\
          &\wedge \satclause(i + 1 - 2n, \tokens) \\
          0.025 & \text{if } i \geq 2n \wedge (\token_{i} \neq ``d") \\
          &\wedge \satclause(i + 1 - 2n, \tokens) \\
          0.1 & \text{if } i \geq 2n \wedge (\token_{i} = ``d") \\
          & \wedge \neg \satclause(i + 1 - 2n, \tokens) \\
          0.225 & \text{if } i \geq 2n \wedge (\token_{i} \neq ``d") \\
          &\wedge \neg \satclause(i + 1 - 2n, \tokens)
	\end{cases}
\end{multline*} 
Here, $\satclause(k, \tokens)$ is a predicate representing the satisfaction of the $k^{\text{th}}$ CNF clause. In particular, there is a straightforward bijection between valid tokenizations and instantiations of the CNF variables, by setting $\variable_i := (\token_{2i} = ``a")$. Then the $k^{\text{th}}$ clause is satisfied iff at least one literal is satisfied,
i.e.:
\begin{equation}
    \satclause(k, \tokens) = 
    \begin{cases}
        \text{True} & \text{if } \exists j. \; (\token_{2\literalidx_{k, j}} = ``a") = \literalsign_{k, j} \\
        \text{False} & \text{otherwise}
    \end{cases}
\end{equation} 

We define $\tokens \models \cnf$ iff $\bigwedge_{k=1}^{K} \satclause(k, \tokens)$, i.e.\ all clauses are satisfied. Now we claim that the 3-CNF formula $\cnf$ is satisfiable iff 
$\max_{\tokens} \llmdist(\tokens, \sentence) > 0.5 (0.45)^{n} (0.9)^{n + K}$. %
We begin by noting that all tokenizations of the string $\sentence$ are of the same length $2n + K$, since each ``abc'' sequence must be split into either (``ab'', ``c'') or (``a'', ``bc''), and the ``ddd...'' sequence must be tokenized into $K$ ``d'' tokens. The probability of any valid tokenization $\tokens$ is thus given by: 
\begin{align*}
    \llmdist(\tokens, \sentence) &= \prod_{i=0}^{2n+K-1} \llmdist(\token_i | \token_1, \dots, \token_{i-1} ) \\
    &=(0.45)^{n} (0.9)^{n}  \\
   & \;\;\;\;\;\;  \prod_{i=2n}^{2n+K-1} \llmdist(\token_i | \token_1, \dots, \token_{i-1} ) \\
    &= (0.45)^{n} (0.9)^{n} \\
    & \;\;\;\;\;\; \prod_{i=2n}^{2n+K-1} \llmdist(``d"| \token_1, \dots, \token_{i-1} )
\end{align*}

Note that the remaining conditional probabilities are all either $0.9$ or $0.025$; thus, $p(\tokens, \sentence) > 0.5 (0.45)^{n} (0.9)^{n + K}$ iff all of these conditional probabilities are $0.9$. Since all of the tokens $\token_{i}$ for $i \geq 2n$ (for $\sentence)$ are ``d'', this happens iff $\tokens \models \cnf$, and the 3-CNF $\cnf$ is satisfiable. Thus \mpe{} is NP-hard.

To show NP-completeness, we note all tokenizations have length $2n + K$ and so oracle calls to the LLM take polynomial time in $n, K$ by Assumption \ref{ass: llmtime}. If the answer to $\mpe$ is Yes, then there exists a tokenization $\tokens'$ with $\llmdist(\tokens, \sentence) > t$ which acts as the certificate. This certificate can be checked in polynomial time; thus \mpe{} is in NP.
\end{proof}

We now move to the problem of computing the marginal probability over all valid tokenizations of $\sentence$. The proof of this result relies on a similar construction to the proof of Theorem \ref{thm:hardmpe}. 

\thmMarg*

\begin{proof}
    Given an instance of \#3-SAT, we construct an instance of \partition{} such that the count of the 3-CNF formula can be easily determined by the marginal string probability.

We define the 3-CNF formula $\cnf$, string $\sentence$, and vocabulary $\vocab$ as in the proof of Theorem \ref{thm:hardmpe}. However, we define a slightly different distribution for the LLM:

\begin{multline*}
	\llmdist(\token_{i}|\token_{0}, ..., \token_{i-1}) = \\
 \begin{cases}
		 0.45 & \text{if } i = 0 \\
   &\wedge (\token_i = ``a" \vee \token_i = ``ab")\\
          0.033 & \text{if } i = 0 \\
          &\wedge \neg(\token_i = ``a" \vee \token_i = ``ab")\\
		 0.9 & \text{if } i < 2n \\
   &\wedge (\token_{i-1} = ``a" \wedge \token_i = ``bc") \\
          0.025 & \text{if } i < 2n \\
          &\wedge (\token_{i-1} = ``a" \wedge \token_i \neq ``bc") \\
		 0.9 & \text{if } i < 2n \\
   &\wedge (\token_{i-1} = ``ab" \wedge \token_i = ``c") \\
          0.025 & \text{if } i < 2n \\
          &(\wedge \token_{i-1} = ``ab" \wedge \token_i \neq ``c") \\
		 0.45 & \text{if } i < 2n \\ 
   &\wedge (\token_{i-1} = ``bc" \vee \token_{i-1} = ``c")  \\
   &\wedge (\token_i = ``a" \vee \token_i = ``ab") \\
          0.033 & \text{if } i < 2n \\
          &\wedge (\token_{i-1} = ``bc" \vee \token_{i-1} = `c")  \\
          &\wedge \neg (\token_i = ``a" \vee \token_i = ``ab") \\
          0.2 & \text{if } i < 2n \wedge (\token_{i-1} = ``d") \\
          1-0.5^{n + K + 1} & \text{if } i \geq 2n 
          \wedge (\token_{i} = ``d") \\
          &\wedge \satclause(i + 1 - 2n, \tokens) \\
          \frac{0.5^{n + K + 1}}{4} & \text{if } i \geq 2n \wedge (\token_{i} \neq ``d") \\
          &\wedge \satclause(i + 1 - 2n, \tokens) \\
          0.5^{n + K + 1} & \text{if } i \geq 2n \wedge (\token_{i} = ``d") \\
          &\wedge \neg \satclause(i + 1 - 2n, \tokens) \\
          \frac{1-0.5^{n + K + 1}}{4} & \text{if } i \geq 2n \wedge (\token_{i} \neq ``d") \\
          &\wedge \neg \satclause(i + 1 - 2n, \tokens)
	\end{cases}
\end{multline*} 

The difference between this distribution and that in Theorem \ref{thm:hardmpe} is the last 4 cases, where the probability is dependent on the number of CNF variables $n$. Now, we will show that the model count of the CNF formula $\cnf$ is equal to $C \in \{0, ..., 2^n\}$ iff $(C - 0.5) (0.45)^{n} (0.9)^{n} < \sum_{\tokens} \llmdist(\tokens, \sentence) < (C + 0.5) (0.45)^{n} (0.9)^{n}$.

As before, the probability of any valid tokenization $\tokens$ is given by: 
\begin{align*}
    \llmdist(\tokens, \sentence) =  &(0.45)^{n} (0.9)^{n} \\& \prod_{i=2n}^{2n+K-1} \llmdist(``d"| \token_1, ..., \token_{i-1} )
\end{align*}
Now, consider the valid tokenizations $\tokens$ of $\sentence$ that correspond to a satisfying assignment of $\cnf$. For any such tokenization, we have that $\satclause(k, \tokens)$ for all $k = 1, ..., K$ and so its probability 
\begin{align*}
    \llmdist(\tokens, \sentence) &= (1-0.5^{n + K + 1})^K (0.45)^{n} (0.9)^{n} \\
    &\geq (1 - 0.5^{n + K + 1}) (0.45)^{n} (0.9)^{n} \\ 
    &\geq (1 - 0.5^{n + 2}) (0.45)^{n} (0.9)^{n}
\end{align*}
On the other hand, for any valid tokenization which does not correspond to a satisfying assignment $\cnf$, there exists a $k$ s.t. $\neg \satclause(k, \tokens)$, and so all such tokenizations have probability $\llmdist(\tokens, \sentence) < \frac{(0.5)^{n + K + 1}}{4} (0.45)^{n} (0.9)^{n}$.

Thus, if $\cnf$ has $C$ satisfying assignments, then we have that 
\begin{align*}
    \sum_{\tokens} \llmdist(\tokens, \sentence) &\geq \sum_{\tokens\models \cnf} \llmdist(\tokens, \sentence) \\
    &= C (1-0.5^{n + K})^K (0.45)^{n} (0.9)^{n} \\
    &\geq C (1 - 0.5^{n + 2}) (0.45)^{n} (0.9)^{n} \\ &> (C - 0.5) (0.45)^{n} (0.9)^{n}
\end{align*} 
where the last inequality follows because $C$ is bounded above by $2^n$. Also, since there are $2^n$ valid tokenizations, we have that 
\begin{align*}
    \sum_{\tokens} \llmdist(\tokens, \sentence) &= \sum_{\tokens \models \cnf} \llmdist(\tokens, \sentence) + \sum_{\tokens \not \models \cnf} \llmdist(\tokens, \sentence) \\
    &< C (0.45)^{n} (0.9)^{n} \\
    & \;\;\;\; + 2^n \frac{(0.5)^{n + K + 1}}{4} (0.45)^{n} (0.9)^{n} \\
    &= C (0.45)^{n} (0.9)^{n} \\
    &\;\;\;\; + (0.5)^{K + 3} (0.45)^{n} (0.9)^{n} \\
    &< (C + 0.5) (0.45)^{n} (0.9)^{n}
\end{align*}
We have shown that the model count of the CNF formula $\cnf$ is equal to $C \in \{0, ..., 2^n\}$ iff $(C - 0.5) (0.45)^{n} (0.9)^{n} < \sum_{\tokens} \llmdist(\tokens, \sentence) < (C + 0.5) (0.45)^{n} (0.9)^{n}$. Given $\sum_{\tokens} \llmdist(\tokens, \sentence)$, we can compute the (unique) value of $C \in \{0, ..., 2^n\}$ for which this holds by binary search, with complexity $O(n)$. Thus we have reduced \#3-SAT to $\partition{}$ and so $\partition{}$ is \#P-hard.

\end{proof}